\def\year{2020}\relax
\documentclass[letterpaper]{article} 
\usepackage{aaai20}  
\usepackage{times}  
\usepackage{helvet} 
\usepackage{courier}  
\usepackage[hyphens]{url}  
\usepackage{graphicx} 
\usepackage{graphicx}  
\usepackage{booktabs}
\usepackage{tikz}
\usetikzlibrary{positioning, arrows.meta}
\usepackage[linesnumbered]{algorithm2e}
\usepackage{amsmath}
\usepackage{amsthm}
\usepackage{paralist}

\newcommand\relphantom[1]{\mathrel{\phantom{#1}}}
\makeatletter
\newcommand*{\indep}{%
  \mathbin{%
    \mathpalette{\@indep}{}%
  }%
}
\newcommand*{\nindep}{%
  \mathbin{
    \mathpalette{\@indep}{/}%
  }%
}
\newcommand*{\@indep}[2]{%
  \sbox0{$#1\perp\m@th$}
  \sbox2{$#1=$}
  \sbox4{$#1\vcenter{}$}
  \rlap{\copy0}
  \dimen@=\dimexpr\ht2-\ht4-.2pt\relax
  \kern\dimen@
  \ifx\\#2\\%
  \else
    \hbox to \wd2{\hss$#1#2\m@th$\hss}%
    \kern-\wd2 %
  \fi
  \kern\dimen@
  \copy0 
}
\makeatother
\newtheorem{defi}{Definition}
\newtheorem{thm}{Theorem}
\newtheorem{lemma}{Lemma}
\newtheorem{cor}{Corollary}

\newtheorem{prop}{Proposition}

\newcommand{\cI}{{\cal I}}
\newcommand{\cF}{{\cal F}}

\title{Recovering Causal Structures from Low-Order Conditional Independencies\thanks{Extended version of paper accepted to the Proceedings of the 34th AAAI Conference on Artificial Intelligence (AAAI-2020).}}
 \author{Marcel Wien\"{o}bst \ and  Maciej Li\'{s}kiewicz \\
 Institute of Theoretical Computer Science, University of L\"{u}beck, Germany\\
 \{wienoebst,liskiewi\}@tcs.uni-luebeck.de
 }

\begin{document}

\maketitle\vspace*{-11mm}

\begin{abstract}
One of the common obstacles for learning causal models from 
data is that high-order conditional independence (CI) relationships 
between random variables are difficult to estimate. 
Since CI tests with conditioning sets of low order can be performed 
accurately even for a small number of observations, a reasonable 
approach to determine casual structures is to base merely on the 
low-order CIs. Recent research has confirmed that, e.g. in the case of 
sparse true causal models, structures learned even from zero- and 
first-order conditional independencies yield good approximations 
of the models. However, a challenging task here is to provide 
methods that faithfully explain a given set of low-order CIs.
In this paper, we propose an algorithm which, for a given set 
of conditional independencies of order less or equal to $k$, 
where $k$ is a small fixed number, computes a faithful graphical 
representation of the given set. Our results complete and generalize 
the previous work on learning from pairwise marginal independencies. 
Moreover, they enable to improve upon the 0-1 graph model which, 
e.g. is heavily used in the estimation of genome networks. 

\end{abstract}

\section{Introduction}
Graphical models, as e.g. directed acyclic graphs (DAGs), 
allow an intuitive and mathematically sound approach to analyze 
complex causal mechanisms \cite{lauritzen1996graphical,Pearl2009}.
Generally, they encode the causal links between variables 
of interests based on conditional independence (CI) 
statements between the variables \cite{Spirtes2000}. Hence, 
the accuracy of estimate of the CIs plays a key role in 
learning graphical models and consequently in causal inference 
from observational data. 

CI testing is a challenging task, particularly in the presence 
of high-order independencies, when the number of variables 
far exceeds the number of observations \cite{Wille06}. 
In such cases, estimations of CIs are usually inaccurate, 
potentially resulting in incorrect links between variables 
in the graphical model.
On the other hand, CI tests with conditioning sets of 
low dimension can be performed accurately even for 
relatively small observed data sets.  
Thus, a natural task is to approximate the true 
causal model using merely low-order CIs.
Recent research in inferring genetic networks
has confirmed the effectiveness of this approach when basing only on zero- and 
first-order independencies \cite{Willeetal01,MagKim01}.

\begin{figure}
  \centering
  \begin{tikzpicture}[xscale = .9]
    \node (a) at (0,0) {$a$};
    \node (b) at (1,0) {$b$}
    edge [<-] (a);
    \node (c) at (0.5, 0.7) {$c$}
    edge [<-] (a)
    edge [->] (b);
    \node (d) at (0.5, -0.7) {$d$}
    edge [<-] (a)
    edge [->] (b);
  \end{tikzpicture}
  \begin{tikzpicture}[xscale = .9]
    \node (a) at (0,0) {$a$};
    \node (b) at (1,0) {$b$}
    edge [<-] (a);
    \node (c) at (0.5, 0.7) {$c$}
    edge [<-] (a)
    edge [->] (b);
    \node (d) at (0.5, -0.7) {$d$}
    edge [->] (a)
    edge [->] (b);
  \end{tikzpicture}
  \begin{tikzpicture}[xscale = .9]
    \node (a) at (0,0) {$a$};
    \node (b) at (1,0) {$b$}
    edge [<-] (a);
    \node (c) at (0.5, 0.7) {$c$}
    edge [->] (a)
    edge [->] (b);
    \node (d) at (0.5, -0.7) {$d$}
    edge [<-] (a)
    edge [->] (b);
  \end{tikzpicture}
  \begin{tikzpicture}[xscale = .9]
    \node (a) at (0,0) {$a$};
    \node (b) at (1,0) {$b$};
    \node (c) at (0.5, 0.7) {$c$}
    edge [<-] (a)
    edge [->] (b);
    \node (d) at (0.5, -0.7) {$d$}
    edge [<-] (a)
    edge [->] (b);
  \end{tikzpicture}
  \begin{tikzpicture}[xscale = .9]
    \node (a) at (0,0) {$a$};
    \node (b) at (1,0) {$b$};
    \node (c) at (0.5, 0.7) {$c$}
    edge [<-] (a)
    edge [->] (b);
    \node (d) at (0.5, -0.7) {$d$}
    edge [->] (a)
    edge [->] (b);
  \end{tikzpicture}
  \begin{tikzpicture}[xscale = .9]
    \node (a) at (0,0) {$a$};
    \node (b) at (1,0) {$b$};
    \node (c) at (0.5, 0.7) {$c$}
    edge [->] (a)
    edge [->] (b);
    \node (d) at (0.5, -0.7) {$d$}
    edge [<-] (a)
    edge [->] (b);
  \end{tikzpicture}
  \caption{All $1$-faithful DAGs for the vertex set $\{a,b,c,d\}$ and the single
  CI statement  $(c \indep d \: | \: a)$.}
  \label{faith:dags:ex}
\end{figure}
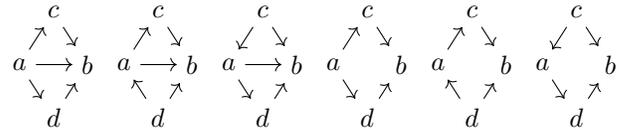

In this paper, we systematically study the problem to extract as much 
``causal knowledge'' as possible from CI statements of order at most $k$,
where $k\ge 0$ is a (typically small) integer. 
More precisely, we investigate the following task:
For a set of variables $V$ and a given set 
$\cI$ of CI statements of the form $(a \indep b \: | \: Z)$, 
with $a,b\in V$, $Z\subseteq V$, and  $|Z|\le k$, find all DAGs $D$ which 
encode up to order $k$ exactly the  CIs in $\cI$, 
i.e., such that for all $a,b,Z$, with $|Z|\le k$,
it is true that $a$ and $b$ are $d$-separated by $Z$ in $D$ if and only if
$(a \indep b \: | \: Z)$ is in $\cI$.
We will call such DAGs \emph{$k$-faithful} to $\cI$ (for formal definitions, 
see Section~\ref{sec:faithful:models}).
Figure~\ref{faith:dags:ex} illustrates all DAGs which are $1$-faithful 
to a single CI statement $(c \indep d \: | \: a)$ for the vertex set $V = \{a,b,c,d\}$.

We observe that this is a generalization of several problems 
already studied in the literature. 
For the simplest case $k=0$, the CI statements are marginal independencies
and the $0$-faithful DAGs are called \emph{faithful to pairwise marginal independencies}.
The problem of deciding if a $0$-faithful DAG exists for a given set of CIs of order zero,
represented as an undirected graph, has been studied in \cite{PearlWermuth94,Idel15}.

Next, with $n$ denoting 
the cardinality of $V$, the problem for $k=n-2$ was first
investigated by \citeauthor{VP92} \shortcite{VP92}. They
called such $k$-faithful DAGs just \emph{faithful} ones and 
presented an algorithm which, for a given 
$\cI$, tests for the existence of a DAG faithful to $\cI$ and produces 
a representation of all such DAGs encoded in form of a completed
partially directed acyclic graph (CPDAG) (we recall all used 
graphical notions in Section~\ref{sec:Preliminaries}).

Note the important difference between the $k$-faithful and faithful
DAGs. Even if $\cI$ consists of CI statements of order $\le k$, 
these two notions differ considerably.
E.g., for the CI statements $\cI$ of order zero and one 
induced by the underlying DAG $D$ shown in Fig.~\ref{comp:to:VP:PC:alg}(a)
the only $1$-faithful DAG is $D$ itself, while no faithful DAG to such $\cI$
exists. This is because for a $1$-faithful DAG, the CIs of order $> 1$
are irrelevant, while a faithful DAG takes that $(x \nindep y \: | \: Z)$,
for all $x,y,$ and $Z$, with $|Z|>1$. 

We also notice that one cannot construct a $k$-faithful DAG
just using a constraint-based structure learning algorithm,
as the SGS or the PC algorithm~\cite{Spirtes2000,KB07}, restricting the 
CI tests to independencies of order $\le k$. For example, 
for the underlying true DAG shown in Fig.~\ref{comp:to:VP:PC:alg}(a) such an approach 
returns a structure with the skeleton given in
Fig.~\ref{comp:to:VP:PC:alg}(b). It is analyzed in detail in
Section~\ref{inc_sec} why the superfluous edge $a-b$ is included in the result of
classical causal structure learning algorithms and
through which rule we are able to remove it.

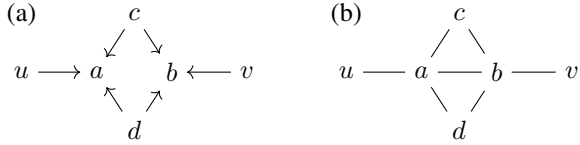
\begin{figure}
  \centering
  \begin{tikzpicture}[xscale = 1,yscale = 1.1]
    \node (aa) at (-1,0.7) {(a)};
    \node (a) at (0,0) {$a$};
    \node (b) at (1,0) {$b$};
    \node (c) at (0.5, 0.7) {$c$}
      edge [->] (a)
      edge [->] (b);
    \node (d) at (0.5, -0.7) {$d$}
      edge [->] (a)
      edge [->] (b);
    \node (u) at (-1,0) {$u$}
      edge [->] (a);
    \node (v) at (2,0) {$v$}
      edge [->] (b);
  \end{tikzpicture}\hspace{7mm}
  \begin{tikzpicture}[xscale = 1,yscale = 1.1]
    \node (aa) at (-1,0.7) {(b)};
    \node (a) at (0,0) {$a$};
    \node (b) at (1,0) {$b$}
      edge [-] (a);
    \node (c) at (0.5, 0.7) {$c$}
      edge [-] (a)
      edge [-] (b);
    \node (d) at (0.5, -0.7) {$d$}
      edge [-] (a)
      edge [-] (b);
    \node (u) at (-1,0) {$u$}
      edge [-] (a);
    \node (v) at (2,0) {$v$}
      edge [-] (b);
  \end{tikzpicture}
  \caption{In (a) the underlying true DAG $D$ is displayed. In (b) we show the
    skeleton of the CPDAG computed by the PC algorithm  
    restricted to CI tests of order zero and one.}
  \label{comp:to:VP:PC:alg}
\end{figure}

\vspace*{-2mm}
\paragraph*{Previous Work.}
Pearl and Wermuth~\shortcite{PearlWermuth94} investigated the problem
whether a set of marginal independencies $\cI$ has a causal
interpretation -- meaning a DAG faithful to $\mathcal{I}$. 
Moreover, they proposed an algorithm to construct a
faithful DAG, but in their
paper they did not give proofs for their theorems. 
\citeauthor{Idel15} \shortcite{Idel15} further considered the stated
problem, characterizing the DAG-representable sets 
by graph-theoretical properties of the marginal independence graphs (these are
undirected graphs with an edge between $a$ and $b$
iff\footnote{We use \emph{iff} as shorthand for \emph{if and only if}.} 
$a$ and $b$ are marginally dependent). 
Additionally, they proposed an algorithm which is based on the 
construction by Pearl and Wermuth~\shortcite{PearlWermuth94}.
However, they did not provide the missing proofs. 

Other works have considered the more general setting which includes
conditional independencies with a singleton conditioning set on top of
marginal independencies. In this context, de Campos and Huete~\shortcite{Campos2000}
introduced the notion of a 0-1 graph. This is an \emph{undirected} graph which contains
an edge $a - b$ iff
$
  (a \nindep b)\  \land \
  [\forall c: \: (a \nindep b \: | \: c)] . 
$
In other words, we obtain the graph by removing all edges between nodes
for which we find an independence of order zero or one.

Wille and B\"{u}hlmann~\shortcite{Wille06} showed that 
-- in the case of graphical Gaussian models --
the 0-1 graphs are good estimators of sparse graphical models and relevant in
biological applications. In particular, they have been used to
model genome networks~\cite{Fuente01,MagKim01,Willeetal01}. 
Later, \citeauthor{castelo2006robust} \shortcite{castelo2006robust} 
generalized the 0-1 graph and the covariance graphs \cite{Cox1993}
to so called $q$-partial graphs.

\vspace*{-2mm}
\paragraph*{Our Results.}
We  provide a constructive solution  to the problem of deciding
if, for a given set $\cI$ of CIs of order less or equal to $k$,
there exists a DAG which is $k$-faithful to $\cI$. We propose an 
algorithm called LOCI (Low-Order Causal Inference) which \textendash\ in case a $k$-faithful DAG exists
\textendash\ outputs all such DAGs encoded in form of a CPDAG. 
This extends and 
generalizes previously known results by \citeauthor{PearlWermuth94} \shortcite{PearlWermuth94} 
as well as by \citeauthor{Idel15} \shortcite{Idel15} 
who provided solutions only for sets of marginal independencies,
i.e. for $k=0$. Moreover, the analysis 
for the correctness of the construction given in this paper, 
fills the gaps in the proofs by \citeauthor{PearlWermuth94}, and by \citeauthor{Idel15}.

The proposed approach also improves some other methods known in the literature 
to learn DAGs from CIs up to a fixed order $k$. In particular, it 
improves the algorithm by 
\citeauthor{Campos2000} \shortcite{Campos2000}
that presupposes knowledge of the topological sorting 
of nodes in the underlying DAG. In contrast, no such knowledge is assumed 
in our algorithm.

\vspace*{-2mm}
\paragraph*{Structure of the paper.} 
In the following section we introduce all preliminary
definitions. Afterwards, in Section~\ref{sec:faithful:models}, we formally
define what faithfulness to a set of CIs means. In Sections~\ref{inc_sec} and~\ref{alg_sec} we derive an
algorithm for finding a compact and faithful representation of a set
of low-order independencies.  We experimentally
compare this algorithm to previous approaches in
Section~\ref{exp_sec}. Finally, we discuss our results in Section~\ref{disc_sec}.
Auxiliary results and most proofs are moved to an appendix.

\section{Preliminaries} 
\label{sec:Preliminaries}
We consider directed and partially directed 
graphs $G = (V,E)$ with $|V| = n$. 
In the latter case, a graph has both directed $a \to b$ and undirected 
$c - d$ edges. 
Two nodes $a$ and $b$ are
called \emph{adjacent} if there is an edge between them (directed or undirected). 
The \emph{degree} of a node $a$ is the number of
nodes adjacent to $a$. 
For an edge $a \rightarrow b$ we call $a$ the
\emph{parent} of $b$ and $b$ the \emph{child} of $a$. A \emph{way} is
a sequence $p_0, \dots, p_t$ of nodes so that for all $i$, with $0
\leq i < t$, there is an edge connecting $p_i$ and $p_{i+1}$. Such a
sequence is called a \emph{path} if $p_i \neq p_j$ holds for all $i,j$,
with $0\le i < j \le t$.
A path from $p_0$ to $p_t$ is called \emph{causal} 
if every edge on the path is directed from $p_i$
towards $p_{i+1}$. A node $b$ is called an \emph{ancestor} of $a$ if there is a
causal path from $b$ to $a$. A node $b$ is called a \emph{descendant} of $a$ if
there is a causal path from $a$ to $b$. $\mathrm{An}_G(a)$ is the set 
of all ancestors of $a$ in graph $G$, $\mathrm{De}_G(a)$ is the set of
all descendants of $a$ in $G$. We use small letters for nodes and values,
and capital letters for sets and random variables.

Of special importance are directed acyclic graphs (DAGs)
containing only directed edges and no directed cycles,
and partially directed acyclic graphs
(PDAGs) that may contain both directed and undirected edges but no directed cycles.
Every DAG is a PDAG.
The \emph{skeleton} of a PDAG $G$ is the undirected
graph where every edge in $G$ is substituted by an undirected edge.

Let $P$ be a joint probability distribution over 
random variables $X_i$, with $i\in V$, and $X$, $Y$ and $Z$ stand
for~any subsets of variables. We use the notation $(X \indep
Y \: | \: Z)_P$ to state that $X$ is independent of $Y$ given $Z$ in $P$. 
A distribution $P$ and a DAG $D=(V,E)$
are called  \emph{compatible} if $D$ factorizes $P$ as 
$ \prod_{i\in V} P(x_i \: | \: \textit{pa}_i)$ over all 
realizations $x_i$ of $X_i$ and $\textit{pa}_i$  
of variables corresponding to the parents of $i$ in $D$.
It is possible to read CIs over $X_i$, with $i\in V$, off a compatible DAG 
through the notion of  $d$-separation.
  Recall, a path $\pi$ is said to be \emph{d-separated} (or \emph{blocked}) by a
  set of nodes $Z$ iff
  (1.) $\pi$ contains a chain $u \rightarrow v \rightarrow w$ or 
    $u \gets v \gets w$ or a
    fork $u \gets v \to v$ such that the middle node
    $v$ is in $Z$, or
  (2.) $\pi$ contains an inverted fork (or \emph{collider}) $u
    \rightarrow v \leftarrow w$ such that the middle node $v$ is not
    in $Z$ and such that no descendant of $v$ is in $Z$.
  A set $Z$ is said to $d$-separate $a$ from $b$ iff $Z$ blocks every
  path from $a$ to $b$.
%
We write $(a \indep b \: | \: Z)_D$ when $a$ and $b$ are $d$-separated by $Z$ in $D$.
Whenever $G$ and $P$ are compatible, it holds for all 
$a, b\in V$, and $Z\subseteq V$, that if $(a \indep b \: | \: Z)_D$ 
then $(X_a \indep X_b \: | \: \{X_i : i\in Z\})_P$.

An inverted fork $u \rightarrow v \leftarrow w$
is called 
a \emph{v-structure} if $u$
and $w$ are not adjacent. A \emph{pattern} of a DAG $D$ is the PDAG
which has the same skeleton as $D$ and which has an oriented
edge $a \to b$ iff there is a vertex $c$,  which is not adjacent to $a$, 
such that 
$c \rightarrow b$ is an edge  in
$D$, too. Essentially, in the pattern of $D$, the only directed edges are the
ones which are part of a v-structure in $D$.

A special case of PDAGs are the so
called CPDAGs~\cite{andersson1997} or completed partially directed
graphs. They represent Markov equivalence classes.
If two DAGs are Markov equivalent, it means that every probability
distribution that is compatible with one of the DAGs is also
compatible with the other~\cite{Pearl2009}.  As shown by Verma and
Pearl~\shortcite{VermaPearl1990} two DAGs
are Markov equivalent \emph{iff} they have the same skeleton and the
same v-structures.

  Given a DAG $D = (V,E)$, the class of Markov equivalent graphs to
  $D$, denoted as $[D]$, is defined as $[D] = \{D' \: | \: D' \text{ is
  Markov equivalent to } D\}$. The graph representing $[D]$ is called
  a CPDAG and is denoted as $D^* = (V,
  E^*)$, with the set of edges defined as follows: $a \rightarrow b$
  is in $E^*$ if $a \rightarrow b$ belongs to every $D' \in [D]$ and
  $a-b$ is in $E^*$ if there exist $D',D'' \in [D]$ so that $a
  \rightarrow b$ is an edge of $D'$ and $a \leftarrow b$ is an edge of
  $D''$. A partially directed graph $G$ is called a CPDAG if $G = D^*$
  for some DAG $D$. 

  \label{ext_def}
  Given a partially directed graph $G$, a DAG $D$ is an extension of
  $G$ iff $G$ and $D$ have the same skeleton and if $a \rightarrow b$
  is in $G$, then $a \rightarrow b$ is in $D$.
  An extension is called \emph{consistent} if additionally $G$ and $D$
  have the same v-structures.
Due to \citeauthor{Meek} \shortcite[Theorem~3]{Meek}, we know that 
when starting with a pattern $G$ of some DAG $D$ and repeatedly executing the
following three rules until none of them applies, we obtain a CPDAG $D^*$
representing the Markov equivalent DAGs:
  \begin{compactenum}
  \item Orient $b - c$ into $b \rightarrow c$ 
    if there is $a
    \rightarrow b$ such that $a$ and $c$ are nonadjacent.
  \item Orient $a-c$ into $a \rightarrow c$ 
    if there is a chain $a
    \rightarrow b \rightarrow c$.
  \item Orient $a - b$ into $a \to b$ 
    if there are two
    chains $a - c \to b$ and $a - d \to b$ such that $c$ and
    $d$ are nonadjacent.
  \end{compactenum}
We will call these three rules the \emph{Meek rules}. 

We note that one obtains the CPDAG $D^*$ by applying the rules   
not only when starting with the pattern of a DAG $D$
but also, more generally, when the initial graph $G$ is any 
PDAG whose consistent extensions form a Markov equivalence class $[D]$.
We will use this property in the correctness proof of the LOCI
algorithm (Algorithm~\ref{alg_gen_pdag}). 


\section{Models Faithful  to CI Statements}
\label{sec:faithful:models}
In this section, we give a formal definition for a $k$-faithful DAG and 
-- for the sake of completeness -- we recall the definitions of a faithful
and a $k$-partial graph. Next, we propose a definition for a compact 
representation of all $k$-faithful DAGs in terms of PDAGs
and show that it yields a CPDAG.

Let $V$ represent the set of variables and $k\ge 0$ be a fixed integer. Let
$\mathcal{I}_V$ be a set of CI statements over variables $X_i$, with $i\in V$, 
given as $(a \indep  b\: |\: Z)$,  with $a,b \in V$ and  $Z \subseteq V$.
Analogously, let $\mathcal{I}_V^k$ be a set of CI statements of order $\le k$, i.e. 
such that $|Z| \leq k$.
For example, the set $\mathcal{I}_V^{0}$ solely contains
marginal independencies.
For a more consistent  notation we 
write $(a \indep b \: | \: Z)_{\mathcal{I}_V^k}$ instead of $(a \indep b \: | \: Z) \in
\mathcal{I}_V^k$, and respectively, $(a \nindep b \: | \:
Z)_{\mathcal{I}_V^k}$ for 
$(a \indep b \: | \: Z) \not\in \mathcal{I}_V^k$. We use an analogous notation for $\cI_V$.
Additionally, in statements like e.g. $(a \indep b \: | \: \{c,d\})$, we omit the brackets 
and write $(a \indep b \: | \: c,d)$.

\begin{defi}[Faithful Graph \cite{VermaPearl1990}]
  \label{def:faithful:dag:VermaPearl1990}
For a set $\mathcal{I}_V$ of CIs, 
a DAG $D=(V,E)$ is called \emph{faithful} to $\mathcal{I}_V$~if 
\[
 \forall (a,b,Z)  \quad
 [(a \indep b \: | \: Z)_{\mathcal{I}_V} \ \Leftrightarrow\  (a \indep b \: | \: Z)_{D} ] .
\]
\end{defi}

\begin{defi}[$k$-Partial Graph \cite{castelo2006robust}]\label{def:k:graph}
For a  set $\mathcal{I}_V^{k}$ of CIs of order $\le k$, 
an undirected graph $G=(V,E)$ is called a \emph{$k$-partial graph} with respect to 
$\mathcal{I}_V^{k}$ if 
\[
 (\forall a,b,Z, |Z|\le k) \quad
 [(a \indep b \: | \: Z)_{\mathcal{I}_V^k} \ \Leftrightarrow\  (a - b \not\in E) ] .
\]
\end{defi}

We will call $k$-partial graphs with $k=1$ also 0-1 graphs, as proposed
by \citeauthor{Wille06} \shortcite{Wille06} who considered such structures 
in the context of graphical Gaussian models.

\begin{defi}[$k$-Faithful Graph]\label{def:faithful:dag}
For a set $\mathcal{I}_V^{k}$ of CIs of order $\le k$, 
a DAG $D=(V,E)$ is called \emph{$k$-faithful} to $\mathcal{I}_V^{k}$~if 
\[
 (\forall a,b, Z, |Z|\le k) \quad
 [(a \indep b \: | \: Z)_{\mathcal{I}_V^k} \ \Leftrightarrow\  (a \indep b \: | \: Z)_{D} ] .
\]
\end{defi}

Due to \citeauthor{VermaPearl1990} \shortcite{VermaPearl1990}, we know that, 
for a given set $\cI_V$, all DAGs faithful to $\cI_V$ can be represented as 
a CPDAG over $V$. A representation of a $k$-partial graph follows straightforwardly 
from the definition.
On the other hand, note that it is not obvious how to represent all DAGs which 
are $k$-faithful to $\cI_V^k$, like e.g. those shown in Fig.~\ref{faith:dags:ex}.

\begin{defi}
A set $\mathcal{I}_V^k$ of CI statements
will be termed \emph{DAG-representable} if there is a
DAG which is $k$-faithful to it. We call a DAG $D$, which is $k$-faithful to
$\mathcal{I}_V^{k}$, \emph{edge maximal} if there is no $k$-faithful DAG whose
edge set is a superset of $D$. Moreover, 
we denote by $\mathcal{F}(\mathcal{I}_V^k)$
the set of all $k$-faithful DAGs to $\mathcal{I}_V^k$. 
\end{defi}
 
For example, for $\cI_V^1=\{(c \indep d \: | \: a)\}$ with 
$V=\{a,b,c,d\}$, Fig.~\ref{faith:dags:ex} shows all DAGs
in $\mathcal{F}(\mathcal{I}_V^k)$.

Below, we define a representation of a set $\mathcal{F}(\mathcal{I}_V^k)$ as a PDAG.
Using our definition, the set of $k$-faithful DAGs from Fig.~\ref{faith:dags:ex}
is represented by the PDAG shown in part (c) of Fig.~\ref{faith:dags:repr:ex}.

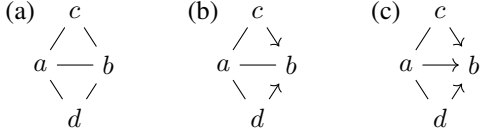
\begin{figure}
  \centering
  \begin{tikzpicture}[xscale = .9]
    \node (a) at (0,0) {$a$};
    \node (b) at (1,0) {$b$}
    edge [-] (a);
    \node (c) at (0.5, 0.7) {$c$}
    edge [-] (a)
    edge [-] (b);
    \node (d) at (0.5, -0.7) {$d$}
    edge [-] (a)
    edge [-] (b);
    \node (a) at (2+0.7,0) {$a$};
    \node (b) at (3+0.7,0) {$b$}
    edge [-] (a);
    \node (c) at (2.5+0.7, 0.7) {$c$}
    edge [-] (a)
    edge [->] (b);
    \node (d) at (2.5+0.7, -0.7) {$d$}
    edge [-] (a)
    edge [->] (b);
    \node (a) at (4+1.4,0) {$a$};
    \node (b) at (5+1.4,0) {$b$}
    edge [<-] (a);
    \node (c) at (4.5+1.4, 0.7) {$c$}
    edge [-] (a)
    edge [->] (b);
    \node (d) at (4.5+1.4, -0.7) {$d$}
    edge [-] (a)
    edge [->] (b);
    \node (l1) at (-0.3, .7) {(a)};
    \node (l2) at (1.7+0.7, .7) {(b)};
    \node (l3) at (3.7+1.4, .7) {(c)};
  \end{tikzpicture}
  \caption{For the example from Fig.~\ref{faith:dags:ex} we show
    the $k$-partial graph (part (a) on the left), the pattern of the
    edge maximal DAGs (part (b) in the middle) and the PDAG
    representing  $\cF(\cI_V^1)$, with $\cI_V^1=\{(c \indep d \: | \: a)\}$ (part (c) on the right).}
  \label{faith:dags:repr:ex}
\end{figure}

We say that
a PDAG $G = (V, E)$ \emph{contains} a set of DAGs
$\{D_i=(V,E_i): i=1,\ldots,t\}$
if for every DAG $D_i = (V, E_i)$ 
it is true that $E_i \subseteq E$. 
Here, we assume 
that an undirected edge $a - b$ in $G$ is
encoded by two directed edges $a \rightarrow b$ and $b \rightarrow a$.
Obviously, a complete undirected graph over $V$ contains 
every set $\mathcal{F}(\mathcal{I}_V^k)$. 
From a causal structure learning perspective, our goal is to
extract from $\cI^k_V$ as much causal knowledge as possible.
We formalize this
goal as to find the minimal PDAG which contains every DAG $k$-faithful to
$\mathcal{I}_V^k$.
In this setting, minimality is considered in
regard to the inclusion relation between the sets of edges.

\begin{defi}
  \label{rep_def}
  A PDAG $G$ \emph{represents} the set $\mathcal{F}(\mathcal{I}_V^k)$
  if $G$ is a minimal graph that contains every graph in
  $\mathcal{F}(\mathcal{I}_V^k)$. 
\end{defi}
  
It is easy to see, that, according to this definition,
the PDAG in part (c) of Fig.~\ref{faith:dags:repr:ex} represents 
the set of $k$-faithful DAGs from Fig.~\ref{faith:dags:ex}.

We note that a PDAG $G$ representing a set
$\mathcal{F}(\mathcal{I}_V^k)$ fulfills the following conditions:
\begin{compactenum}
\item There is an edge $a - b$ in $G$ iff DAGs $D, D' \in
  \mathcal{F}(\mathcal{I}_V^k)$ exist such that there is an edge $a
  \rightarrow b$ in $D$ and an edge $a \leftarrow b$ in $D'$.
\item There is an edge $a \rightarrow b$ in $G$ iff a DAG
  $D \in \mathcal{F}(\mathcal{I}_V^k)$ exists which contains the
  edge $a \rightarrow b$ and no DAG in
  $\mathcal{F}(\mathcal{I}_V^k)$ contains the edge $a \leftarrow
  b$.
\item There is no edge between $a$ and $b$ in $G$ iff no DAG in
  $\mathcal{F}(\mathcal{I}_V^k)$ contains an edge between $a$ and $b$.
\end{compactenum}
From this perspective 
one can already view the representation $G$ as a generalization of the
notion of a CPDAG that is used to represent Markov equivalent DAGs 
of the same skeleton. Note that DAGs in $\cF(\cI_V^k)$ can have 
different skeletons. Interestingly, we prove that the PDAG representing 
the set of $k$-faithful graphs is still a CPDAG. 
\begin{prop}\label{prop:represantation}
For a given set $\cI_V^k$ of CIs, the representation $G$ of 
all $k$-faithful DAGs $\mathcal{F}(\mathcal{I}_V^k)$ is a CPDAG.
Moreover, any consistent extension of $G$ is a DAG $k$-faithful to $\cI_V^k$. 
\end{prop}
In particular, this means that the representation $G$ is itself a
faithful model of all CIs up to order $k$.

\section{Determining the Skeleton} 
\label{inc_sec}
For a given set $\mathcal{I}_V^k$ of conditional independence statements  
up to order $k$, our goal is to find the
representation of the set of $k$-faithful DAGs
$\mathcal{F}(\mathcal{I}_V^k)$. By definition, this is the minimal graph which contains
every $k$-faithful DAG. Thus, our strategy is the following. Starting with
the complete graph, we want to remove all edges which 
do not belong to any $k$-faithful DAG and, vice versa, keep all edges which are in at least one
$k$-faithful DAG. This is in line with the paradigm of
\emph{constraint-based} causal structure learning.

In this section, we characterize all pairs of nodes which are
nonadjacent in \emph{every} $k$-faithful DAG. These pairs of nodes are exactly
the ones which are nonadjacent in the representation as well. This
means that, by finding them, we can construct the skeleton of the representation.
We will explore how edges are oriented in the subsequent section.

One setting in which two nodes have to be nonadjacent is quite obvious. 
If we have a statement $(a \indep b \: | \: Z)_{\mathcal{I}_V^k}$, 
it follows trivially that there cannot be an edge between $a$ and $b$ 
in any $k$-faithful DAG. However, as we will see, though this condition is necessary,
it is not sufficient for the non-adjacency of vertices in $k$-faithful DAGs.
As the main result of this section, we provide a property between two nodes
(we call it \emph{incompatibility}) and using this property we formulate 
a criterion for non-adjacency which is both necessary and sufficient  
(Proposition~\ref{skel_prop}). The incompatibility between two nodes $a$ and $b$
expresses some higher order conditional independencies 
which can be derived from CI statements up to order $k$.

\subsection{Derivation of Higher-Order CI Statements}   
When having
access to \emph{all} conditional independencies without a restriction
on the order, removing
edges corresponding to these known CIs is sufficient for learning the skeleton
of the underlying causal structure. For example, the SGS and the PC
algorithm~\cite{Spirtes2000} work exactly in this fashion. However,
only removing these edges is not sufficient even for obtaining the
skeleton of the representation (or the skeleton of a $k$-faithful DAG) 
when we consider \emph{order-bounded}
sets of independencies. We will now investigate why this is the
case and show how this obstacle can be overcome.

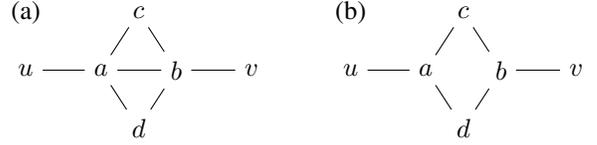
\begin{figure}
  \centering
  \begin{tikzpicture}[xscale = 1,yscale = 1.1]
    \node (aa) at (-1,0.7) {(a)};
    \node (a) at (0,0) {$a$};
    \node (b) at (1,0) {$b$}
      edge [-] (a);
    \node (c) at (0.5, 0.7) {$c$}
      edge [-] (a)
      edge [-] (b);
    \node (d) at (0.5, -0.7) {$d$}
      edge [-] (a)
      edge [-] (b);
    \node (u) at (-1,0) {$u$}
      edge [-] (a);
    \node (v) at (2,0) {$v$}
      edge [-] (b);
  \end{tikzpicture}\hspace{7mm}
  \begin{tikzpicture}[xscale = 1,yscale = 1.1]
    \node (aa) at (-1,0.7) {(b)};
    \node (a) at (0,0) {$a$};
    \node (b) at (1,0) {$b$};
    \node (c) at (0.5, 0.7) {$c$}
      edge [-] (a)
      edge [-] (b);
    \node (d) at (0.5, -0.7) {$d$}
      edge [-] (a)
      edge [-] (b);
    \node (u) at (-1,0) {$u$}
      edge [-] (a);
    \node (v) at (2,0) {$v$}
      edge [-] (b);
  \end{tikzpicture}
  \caption{Left (a): The 0-1 graph for the CI statements $\cI_V^1$
    induced by the underlying DAG $D$ shown in Fig.~\ref{comp:to:VP:PC:alg}(a). 
    It contains the edge $a-b$, as there is
    no independence of order zero or one between these nodes.
    Right (b): The skeleton of the graph computed by our algorithm (presented in the next section).
    Nodes $a$ and $b$ are not incident since 
    they are incompatible (according to our definition). 
    The justification is that, as seen in Fig.~\ref{comp:to:VP:PC:alg}(a), there is an independence
    of order two $(a \indep b \: | \: c,d)$. }
  \label{confl_ex_01:new}
\end{figure}

The outlined problem is illustrated in
Fig.~\ref{confl_ex_01:new}, for the CI statements of order $0$ or $1$
induced by the underlying DAG $D$ shown in Fig.~\ref{comp:to:VP:PC:alg}(a), i.e.
for the set  $\mathcal{I}_V^1 = \{(u \indep c), (u \indep d), (u
\indep b), (u \indep v), (a \indep v), (c \indep v), (d \indep v), (c
\indep d), (u \indep c \: | \: v), \dots\}$ of all zero- and first-order
independencies found in this DAG. Choosing the value $k = 1$ allows us a comparison with
0-1 graphs, but such an example can be constructed for all $0 \leq k < n-2$.
In part (a) we show the corresponding 0-1 graph. This graph is
constructed 
using the simple strategy of removing an
edge if a zero- or first-order independence is
present. We see that the nodes $a$ and $b$ are adjacent in this graph
because no independence $(a \indep b \: | \: Z)$ of order zero or one exists. However, in the
underlying DAG $D$ the nodes $a$ and $b$ are nonadjacent. Moreover, it is impossible
to find a $k$-faithful DAG which contains the edge $a \rightarrow b$ or 
$a \leftarrow b$. In fact, $D$ is the only $k$-faithful DAG. 
Essentially, this is the case because the edge
between $a$ and $b$ (if present) would need to be in two conflicting
v-structures, namely $u \rightarrow a \leftarrow b$ and $a \rightarrow
b \leftarrow v$, to make sure that $u$ and $b$ as well as $a$ and $v$
are marginally independent. This is clearly impossible. 
From the fact that no $k$-faithful DAG contains an edge between $a$ and
$b$, we can infer that there has to be a higher-order
CI between $a$ and $b$. Here, this higher-order CI is $(a \indep b \: | \: c,d)$. 

It should be noted that de Campos and Huete~\shortcite{Campos2000} already
discovered that it is possible to remove further edges from 0-1
graphs (they considered a similar example in Fig.~2 of their
paper). However, their method for deleting such edges relied on
the topological ordering in the underlying DAG and they did not
classify these edges. Requiring the topological ordering is a
large obstacle for practical applications. Our proposed methods do not
rely on the topological sorting as we give a simple classification of the edges
that have to be removed from the 0-1 graph in order to obtain the
skeleton of the representation and, by that, the skeleton of a $k$-faithful DAG. 

We will now formalize the situation
just described in the following definition and thereby introduce the so called
\emph{incompatible nodes}:
\begin{defi}
  \label{confl_def}
   Let $\mathcal{I}_V^k$ be a set of CIs of order $\le k$.
  Then two nodes $a$ and $b$ are
  called \emph{incompatible} iff there exist $u, v, S, T$ such that the following two conditions hold:
  \begin{enumerate}
    \item $(u \indep b \: | \: S)_{\mathcal{I}_V^k}  \land (u \nindep a \: | \: S)_{\mathcal{I}_V^k} 
      \land (a \nindep b \: | \: S)_{\mathcal{I}_V^k} \land a \not\in S$,
    \item $(v \indep a \: | \: T)_{\mathcal{I}_V^k} \land  (v \nindep b \: | \: T)_{\mathcal{I}_V^k} 
    \land (b \nindep a \: | \: T)_{\mathcal{I}_V^k} \land b \not\in T.$
  \end{enumerate}  
\end{defi}
We can see that the nodes $a$ and $b$ in the example in Fig.~\ref{confl_ex_01:new} are incompatible because
$(u \indep b)_{\mathcal{I}_V^1}$, $(u \nindep a)_{\mathcal{I}_V^1}$, 
$(a \nindep b)_{\mathcal{I}_V^1}$, $(v \indep a)_{\mathcal{I}_V^1}$,
and $(v \nindep b)_{\mathcal{I}_V^1}$
hold. In this case, $S$ and $T$ are both the empty set. It follows
immediately that $a \not\in S$ and $b \not\in T$ are satisfied. Moreover, $(b
\nindep a)_{\mathcal{I}_V^1}$ follows by symmetry from $(a \nindep
b)_{\mathcal{I}_V^1}$.

We now prove formally that if the nodes $a$ and $b$ are incompatible, 
there cannot be an edge between $a$ and $b$ in any $k$-faithful DAG.
Firstly, we show the following:
\begin{lemma}
  \label{early_dir_thm}
  Let $\mathcal{I}_V^k$ be a set of CIs of order $\le k$.
  If we have $(u \indep b \: | \: Z)_{\mathcal{I}_V^k}$, $(u \nindep a \:
  | \: Z)_{\mathcal{I}_V^k}$ and $a \not\in Z$, then no DAG $k$-faithful to
  $\mathcal{I}_V^k$ contains the edge $a \rightarrow b$.
\end{lemma}
\begin{proof}
  Assume, 
  there is an edge $a \rightarrow
  b$ in a $k$-faithful DAG $D$.  In this DAG, $(u \nindep a \; | \; Z)_{\mathcal{I}_V}^k$ has to
  hold. This means that there is a path between $u$ and $a$ which is
  not blocked by $Z$.  But as we have the edge $a \rightarrow b$
  in $G$, there will also be a path between $u$ and $b$ which is not
  blocked by $Z$ (note that 
  $a \not\in Z$).  A contradiction.
\end{proof}
We  immediately conclude that incompatible nodes cannot be adjacent in any
$k$-faithful DAG:
\begin{cor}
  \label{no_pres_cor}
  Let $\mathcal{I}_V^k$ be a set of CIs of order $\le k$.
  If the nodes $a$ and $b$ are
  incompatible, they are nonadjacent in every DAG
  $k$-faithful to $\mathcal{I}_V^k$. 
\end{cor}
Due to the conditions stated in the definition of incompatible nodes
(Definition~\ref{confl_def}), it follows from
Lemma~\ref{early_dir_thm} that neither the edge $a \rightarrow b$
nor $a \leftarrow b$ can be in any $k$-faithful DAG.

\subsection{A Complete Criterion for Adjacency}
The following proposition underlines the importance of the notion of incompatible
nodes by showing that making such nodes nonadjacent in a 0-1 graph is not only
necessary, but also sufficient in order to obtain the skeleton of the representation:
\begin{prop}
  \label{skel_prop}
  In the 
  representation of
  $\mathcal{F}(\mathcal{I}_V^k)$ two nodes $a$ and $b$ are adjacent
  if and only if
  \begin{compactenum}
  \item[$(i)$] there is no 
  CI $(a \indep b \: | \: Z)_{\mathcal{I}_V^k}$
  for $Z \subseteq V$, $ |Z| \leq k$, and 
  \item[$(ii)$] the nodes $a$ and $b$ are not incompatible.
  \end{compactenum}
\end{prop}
This result stems from the correctness proof of the LOCI algorithm presented in the
following section (Algorithm~\ref{alg_gen_pdag}). There, we complete the
construction of the representation by showing how
edges can be oriented. 

\section{Determining the Faithful Model} 
\label{alg_sec}
Now we are ready to discuss how to find the representation which
$k$-faithfully models the CIs given in the set $\mathcal{I}_V^k$. This
will also enable us to decide if $\mathcal{I}_V^k$ even has a causal
explanation. To answer this question, we attempt to construct the
representation and if this fails, conclude that there can be no $k$-faithful DAG.

\begin{algorithm}
  \caption{The LOCI algorithm  computes 
    the representation $G$ for a DAG-representable set
    of CIs up to order $k$. Note that we represent an undirected edge 
    $a - b$ as a pair $a \to b$ and $a \gets b$.}
  \label{alg_gen_pdag}
  \DontPrintSemicolon
  \SetKwInOut{Input}{input}\SetKwInOut{Output}{output}
  \SetKwFor{Rep}{repeat}{}{end}
  \Input{Vertex set $V$, DAG-representable set $\mathcal{I}_V^{k}$ of CIs with order $\leq k$}
  \Output{CPDAG $G$ representing $\mathcal{F}(\mathcal{I}_V^k)$} 
  Form the graph $G$ on the vertex set $V$ which has an undirected
  edge $a - b$ if 
  for every subset $Z$ of $V$, with $|Z| \leq k$, it is true $(a \nindep
  b \: | \: Z)_{\mathcal{I}_V^k}$. \;\label{alg_gen_start_line}
  \ForEach{CI $(a \indep b \: | \: Z)$ in $\mathcal{I}_V^k$ and every
    $c \in V \backslash\{a, b\}$ \label{start_for_allk}}{
    \If{$(a \nindep c \: | \:
      Z)_{\mathcal{I}_V^{k}}$, $(c \nindep b \: | \:
      Z)_{\mathcal{I}_V^{k}}$ and $c \not\in Z$ \label{cond_gen_line}}
    {Remove $a \leftarrow c$ and $c \rightarrow b$ from $G$. \; \label{alg_gen_dir_line}}}
  \label{done_line}
  \Rep{the Meek rules until no rule can be applied. \label{meek}}{ 
    1. \begin{tikzpicture}[baseline=-0.2em]
      \node (a) at (-0.75,0) {$a$};
      \node (b) at (0,0) {$b$};
      \node (c) at (0.75,0) {$c$}; 
      \draw [->](a) -- (b);
      \draw [-](b) -- (c);
      \node (r) at (1.5,0) {$\Rightarrow$};
      \node (a) at (2.25,0) {$a$};
      \node (b) at (3,0) {$b$};
      \node (c) at (3.75,0) {$c$}; 
      \draw [->](a) -- (b);
      \draw [->](b) -- (c);
    \end{tikzpicture} \;
    2. \begin{tikzpicture}[baseline=-.2em]
    \node (a) at (-0.75,0) {$a$};
    \node (b) at (0,0) {$b$};
    \node (c) at (0.75,0) {$c$}; 
    \draw [->](a) -- (b);
    \draw [->](b) -- (c);
    \draw [-] (a) edge [bend left=25] (c);
    \node (r) at (1.5,0) {$\Rightarrow$};
    \node (a) at (2.25,0) {$a$};
    \node (b) at (3,0) {$b$};
    \node (c) at (3.75,0) {$c$}; 
    \draw [->](a) -- (b);
    \draw [->](b) -- (c);
    \draw [->] (a) edge [bend left=25] (c);
  \end{tikzpicture} \;
    3. \begin{tikzpicture}[baseline=-.2em]
    \node (a) at (-0.75,0) {$a$};
    \node (c) at (0,0.3) {$c$};
    \node (d) at (0,-0.35) {$d$};
    \node (b) at (0.75,0) {$b$}; 
    \draw [-](a) -- (b);
    \draw [->](c) -- (b);
    \draw [->](d) -- (b);
    \draw [-] (a) -- (c);
    \draw [-] (a) -- (d);
    \node (r) at (1.5,0) {$\Rightarrow$};
    \node (a) at (2.25,0) {$a$};
    \node (c) at (3,0.35) {$c$};
    \node (d) at (3,-0.35) {$d$};
    \node (b) at (3.75,0) {$b$}; 
    \draw [->](a) -- (b);
    \draw [->](c) -- (b);
    \draw [->](d) -- (b);
    \draw [-] (a) -- (c);
    \draw [-] (a) -- (d);
  \end{tikzpicture}  \; 

    } \label{end_meek}
\end{algorithm}

The LOCI (Low-Order Causal Inference) algorithm for constructing the
representation is presented as Algorithm~\ref{alg_gen_pdag}. 
We note that it works for arbitrary values $k$, 
in particular for $k = 0$ for which the CI statements 
represent marginal independencies.
The algorithm  can be divided 
into three stages.

In the first stage (line~\ref{alg_gen_start_line}),
the $k$-partial graph is generated which
can be constructed by removing, from the complete undirected graph, edges corresponding 
to a CI in $\mathcal{I}_V^k$. 
We remark that, in general, one does not obtain 
this graph 
by executing e.g. the ``skeleton phase'' of the
PC algorithm~\cite{Spirtes2000} up to order $k$.
Here, only separating
sets formed by the adjacent nodes are considered. Therefore, some
separating sets of order $\leq k$ can be overlooked. Instead, it is
necessary to consider all possible separating sets $Z$ up to order
$k$.

In the second stage (lines~\ref{start_for_allk}
to~\ref{done_line}), directed edges are removed according
to the rule in Lemma~\ref{early_dir_thm}. Recall that an undirected edge 
$u - v$ is represented as a pair $u \to v$ and
$u \gets v$. Thus, removing only the edge $u \to v$ 
means the orientation of $u - v$ into $u \gets v$. 
Obviously, removing both directed edges denotes the deletion of the edge $u - v$.
The aim of the second stage is (1) to remove the remaining undirected edges which 
do not satisfy the criterion in Proposition~\ref{skel_prop}, i.e. 
the edges $u - v$ between incompatible nodes, and (2) to determine all 
v-structures. We note that in this stage, the algorithm also orients 
some further edges, which are not involved in v-structures. 

To prove the correctness, we make use of the fact
that we can always apply Lemma~\ref{early_dir_thm} to triples
$a, b, c$ (used in lines~\ref{start_for_allk} to~\ref{done_line}) 
and through this delete two directed edges $a \gets c$ and $c \to b$ at the same time.
This step ensures 
that all incompatible nodes are nonadjacent.
In particular, nodes $a$ and $c$ are incompatible 
iff the edges $a \gets c$ and $a \to c$ are removed 
at the different steps of the iteration corresponding to triples
$a,b,c$ and $\hat{a},c,a$ in lines~\ref{start_for_allk} to~\ref{done_line}:
Indeed, the two conditions (cf. Definition~\ref{confl_def})
\begin{enumerate}
  \item $(a \indep b \: | \: Z)_{\mathcal{I}_V^k}   
     \land (a \nindep c \: | \: Z)_{\mathcal{I}_V^k}  
     \land (c \nindep b \: | \: Z)_{\mathcal{I}_V^k}  
     \land c\notin Z,$
  \item $(\hat{a} \indep c \: | \: \hat{Z})_{\mathcal{I}_V^k}  
     \land (\hat{a} \nindep a \: | \: \hat{Z})_{\mathcal{I}_V^k}  
     \land (a \nindep c \: | \: \hat{Z})_{\mathcal{I}_V^k}
     \land a\notin \hat{Z}$
\end{enumerate}
are true iff the algorithm removes $a \gets c$ and $a \to c$
in line~\ref{alg_gen_dir_line}.
If, however, only the edge $a \gets c$ is removed from the
undirected edge $a - c$, the edge $a \to c$ remains, meaning 
the orientation of $a - c$ into $a \to c$. 

Moreover,   
we show that in stage two all v-structures of the representation are
oriented. Note that, in order to make sure all v-structures $x \rightarrow y
\leftarrow z$ are correctly oriented, even if $x$ and $z$ are
incompatible, it is necessary to consider all triples of nodes $a, b,
c$ and not only chains $a - c - b$ as in common causal structure
learning algorithms like the PC-algorithm~\cite{Spirtes2000}.

Finally, in the third stage (line~\ref{meek} to~\ref{end_meek}),
the algorithm orients further undirected edges 
through the Meek rules.
The graph obtained after completing stage two already characterizes a Markov
equivalence class, as the skeleton and the v-structures are determined. In order to obtain the representation, we have to
maximally extend it into a CPDAG. This is why we are able to apply the Meek rules.

Before stating the main results, 
we illustrate how the LOCI algorithm works  using
as an example instance the zero- and first-order independencies $\mathcal{I}_V^1
= \{(c \indep d \: | \: a)\}$ over $V = \{a,b,c,d\}$, that have been 
discussed in Fig.~\ref{faith:dags:ex} and~\ref{faith:dags:repr:ex}. 
In (a), (b) and (c) of Fig.~\ref{faith:dags:repr:ex} 
the graph $G$ is shown after completing stage one, two, and three, respectively. Thus,
in (a) there is no edge between $c$ and $d$ as we have the
independence $(c \indep d \: | \: a)$ in $\mathcal{I}_V^k$, while all other
edges are present. In (b) we see that the edges $c \rightarrow b$ and
$d \rightarrow b$ are oriented. Essentially, there can be no edge $c
\leftarrow b$ (or $d \leftarrow b$) as in that case $(c \indep d \: |
\: a)_{\mathcal{I}_V^k}$ cannot hold without a collider at node $b$. In this regard, 
stage two is similar to the orientation of
v-structures in the SGS or PC algorithm~\cite{Spirtes2000}. The
difference is, however, as emphasized before, that in the LOCI algorithm further
nodes can be separated during this stage. An example for this are the
incompatible nodes $a$ and $b$ of the example in
Fig.~\ref{comp:to:VP:PC:alg} and~\ref{confl_ex_01:new}.
We, moreover,
remark that, while all v-structures are detected, the result is not
always a pattern, as it is possible that even
further edges are already oriented. 
Finally, in part (c) the resulting graph $G$ is
shown. Here, the edge $a - b$ has been oriented into $a \rightarrow b$
due to the third Meek rule. As seen in Fig.~\ref{faith:dags:ex}, there are six DAGs which are $k$-faithful to
$\mathcal{I}_V^1$. Three of them contain the edge $a \rightarrow b$
and in the other three $a$ and $b$ are nonadjacent. However, the edge $a
\leftarrow b$ is in no $k$-faithful DAG which is why the orientation $a
\rightarrow b$ is correct.

We now state the main result of this paper that
the LOCI algorithm produces the required representation:
\begin{thm}
  \label{thm_alg_correct}
  The graph $G$ resulting from the LOCI algorithm (Algorithm~\ref{alg_gen_pdag}) is the
  representation of the set $\mathcal{F}(\mathcal{I}_V^k)$, if $\mathcal{I}_V^k$ 
  is DAG-representable.
\end{thm}
Some ingredients of the proof of this theorem 
have already been stated in this and the previous section.
The complete proof can be found in the appendix. 

The result enables us to \emph{decide} 
whether a given set $\mathcal{I}_V^k$ has a causal
explanation. This is possible through the following approach: We can apply the LOCI algorithm to
$\mathcal{I}_V^k$ and check whether the resulting graph is a $k$-faithful
CPDAG. If it is, clearly there is a causal explanation of
$\mathcal{I}_V^k$, namely the produced graph (Proposition~\ref{prop:represantation}). 
If it is not, then $\mathcal{I}_V^k$ cannot have such a causal
explanation as, if this were the
case, $G$ would be the representation (Theorem~\ref{thm_alg_correct})
and, therefore, as argued above, a faithful model. Thus, we conclude:

\begin{prop}\label{prop:decision}
There exists an algorithm which for a given set $\cI_V^k$ of CIs
tests if the set is DAG-representable.
\end{prop}

\section{Experimental Analysis}
\label{exp_sec}
The representation $G$ of a set $\mathcal{I}_V^k$ is in itself a very
useful graph as it faithfully models the CIs of order $\leq
k$. But apart from this, it can also be used as an approximation of the true
underlying causal structure. It can even be argued that it is the best
approximation obtained through the given conditional independence
information. Because of the minimality of the representation, removing a further edge from $G$ would
mean that some DAG $k$-faithful to $\mathcal{I}_V^k$ is not contained in
it anymore.

Thus, we investigate in this section how well the representation $G$ of a set of low-order CIs
is able to capture the underlying true causal structure. We do this
experimentally by generating a sparse DAG which we then try to recover
with the LOCI algorithm. We confine our analysis to the case
$k=1$ which allows us a comparison with the 0-1 graph
model. For this, we compare the number of adjacencies (meaning the number
of edges in the skeleton) in the 0-1 graph, the CPDAG $G$ and the true
DAG. This enables us, in particular, to estimate the influence of
removing edges between incompatible nodes. Additionally, we investigate
how many v-structures from the true DAG can already be found in the
CPDAG $G$, giving us an indication how well the edge orientations
are captured in the representation.

We begin by explaining how we generated the set of independencies
$\mathcal{I}_V^1$. An undirected graph with $n$ nodes is drawn randomly. More precisely,
each edge is present with probability $d / (n-1)$, meaning every node has expected degree $d$.
Afterwards, a topological ordering of the nodes is
randomly chosen in order to obtain a DAG $D$ from the generated
graph. From this DAG we can read off
all zero- and first-order independencies through the notion of
d-separation and thereby produce the set $\mathcal{I}_V^1$ needed for
the LOCI algorithm. 

The representation $G$ can be obtained by performing
the LOCI algorithm on $\mathcal{I}_V^1$ and the 0-1 graph can
be easily obtained as well by removing edges which correspond to independencies in
$\mathcal{I}_V^1$. Through the generation procedure we also have
access to the underlying true DAG $D$. First, we look at the number of
adjacencies in the different graphs (see Table~\ref{table_adj}). The
displayed numbers are the means of 100 independent trials
\begin{table}
  \centering
  \begin{tabular}{rrrrr} \toprule 
    \multicolumn{2}{c}{DAG} & \multicolumn{3}{c}{Number of edges}
    \\ \cmidrule(l){1-2} \cmidrule(l){3-5}
    $n$   & $d$ & 0-1         & skel.\ $G$ & skel.\ $D$  \\ \midrule
    $20$  & $2$ & $27{.}21$      & $25{.}43$      & $19{.}81$  \\
    $20$  & $3$ & $57{.}88$      & $51{.}07$      & $29{.}79$  \\
    $20$  & $4$ & $96{.}57$      & $87{.}47$      & $40{.}03$  \\
    $20$  & $5$ & $126{.}73$     & $119{.}49$     & $50{.}21$  \\ \midrule
    $60$  & $2$ & $77{.}85$      & $69{.}69$      & $58{.}97$  \\
    $60$  & $3$ & $226{.}03$     & $160{.}43$     & $88{.}55$  \\
    $60$  & $4$ & $512{.}89$     & $346{.}06$     & $118{.}29$ \\
    $60$  & $5$ & $820{.}69$     & $579{.}67$     & $148{.}08$ \\ \midrule
    $100$ & $2$ & $125{.}87$     & $113{.}37$     & $99{.}40$  \\
    $100$ & $3$ & $413{.}68$     & $266{.}92$     & $149{.}42$ \\
    $100$ & $4$ & $1{,}061{.}39$ & $598{.}11$     & $199{.}34$ \\
    $100$ & $5$ & $1{,}905{.}20$ & $1{,}118{.}60$ & $248{.}92$ \\
    \bottomrule
  \end{tabular}
  \caption{We consider random DAGs with $n$ nodes and expected node degree $d$. This means each edge is present with
    probability $d / (n-1)$.  We present the number of edges in the 0-1
    graph, the skeleton of representation $G$ and the skeleton of the true
    DAG $D$. All values are the means of 100 independent trials.}
  \label{table_adj}
\end{table}
and we consider graphs with $20$, $60$ and $100$ nodes and expected node
degree $2$, $3$, $4$ and $5$. Clearly, the numbers are nonincreasing
from left to right. To be more precise, it holds that $A_{D} \subseteq
A_{G} \subseteq A_{\text{0-1}}$ where $A$ is the set of all adjacencies. This is
due to the fact that every $k$-faithful DAG is contained in $G$ and that
$G$ is constructed by removing edges from the 0-1 graph.

We begin the analysis by exploring how close $G$ is to the true causal structure.
It can be seen that, in particular for larger graphs, we are only able
to reasonably estimate the underlying structure up to expected degree
$3$. For example for $n = 100$ and $d = 4$, the representation $G$
contains almost three times as many adjacencies as $D$. For $n = 100$
and $d = 2$, the estimation is very close to the true DAG and even for $d = 3$
the ratio between the number of adjacencies in $G$ and $D$ is quite reasonable, being well below
two. Notably, in the latter setting the improvement over the 0-1
graphs is significant. Actually, the difference in the number of adjacencies
is larger between the 0-1 graph and $G$ than between $G$ and $D$. More
generally, we see that for larger graphs the gap between the 0-1 graph and
$G$ is substantial, meaning there is a great number of incompatible
nodes. This underlines the importance of removing edges between such
nodes in order to find a graphical model which is $k$-faithful to a set of
independencies $\mathcal{I}_V^1$. 
We can conclude that it is possible to estimate the true causal
structure reasonably well, given that it is sparse. Moreover, it is
crucial to remove the edges between incompatible nodes.
But, apart from the
adjacencies (or in other words the skeleton), the representation also
contains directed edges and, thus, also
v-structures. Therefore, it is interesting to investigate how many
v-structures from the true DAG can already be found in $G$. These
numbers are presented in Table~\ref{table_vs}. Here, we show the
number of v-structures in $D$, in $G$ and those which are in both $D$
and $G$. For better readability, the numbers are normalized by the
number of nodes $n$.
We consider the same setting as above. 

\begin{table}
  \centering
  \begin{tabular}{rrrrr} \toprule 
    \multicolumn{2}{c}{DAG} & \multicolumn{3}{c}{Number of v-s per node}
    \\ \cmidrule(l){1-2} \cmidrule(l){3-5}
    $n$   & $d$ & v-s in $G$    &  v-s in $D$ & v-s in both \\ \midrule
    $20$  & $2$ & $1{.}427$       & $0{.}561$  & $0{.}552$  \\
    $20$  & $3$ & $5{.}276$       & $1{.}190$  & $1{.}111$  \\
    $20$  & $4$ & $10{.}784$      & $2{.}024$  & $1{.}649$  \\
    $20$  & $5$ & $13{.}556$      & $2{.}938$  & $2{.}031$  \\ \midrule
    $60$  & $2$ & $1{.}445$       & $0{.}614$  & $0{.}612$  \\
    $60$  & $3$ & $10{.}904$      & $1{.}343$  & $1{.}317$  \\
    $60$  & $4$ & $42{.}979$      & $2{.}369$  & $2{.}200$    \\
    $60$  & $5$ & $90{.}279$      & $3{.}650$  & $3{.}119$  \\ \midrule
    $100$ & $2$ & $1{.}383$       & $0{.}654$  & $0{.}653$ \\
    $100$ & $3$ & $12{.}982$      & $1{.}441$  & $1{.}424$ \\
    $100$ & $4$ & $59{.}859$      & $2{.}521$  & $2{.}419$  \\
    $100$ & $5$ & $161{.}390$     & $3{.}872$  & $3{.}504$ \\
    \bottomrule
  \end{tabular}
  \caption{In the same setting as in Table~\ref{table_adj} we present
    the number of v-structures (v-s, for short) in $G$, in $D$ and those in both graphs.}
  \label{table_vs}
\end{table}

We investigate first how many v-structures are in both $G$ and
$D$ compared to the number of v-structures in $D$. This shows how many
of the v-structures of the true underlying DAG
the LOCI algorithm is able to detect. We
can see that almost all v-structures are found even for
larger expected node degrees $4$ or $5$. E.g. for $n = 100$ and $d =5$,
the LOCI algorithm discovers $3{.}504$ out of $3{.}872$
v-structures per node.

While the LOCI algorithm finds most of the v-structures in $D$, we can see that
there are many more additional v-structures in the representation $G$. While this is in
reasonable limits for sparse graphs (for $d=2$ we see roughly a
doubling of the number of v-structures), the difference is much more extreme in
denser graphs. In particular, for $n = 100$ and $d = 5$ there are
$161{.}39$ v-structures per node in $G$ and only $3{.}872$ in $D$.
This is due to the fact that, as we have seen in Table~\ref{table_adj}, there are many more
edges in $G$. At first glance, however, the increase in v-structures
is much more extreme (a factor more than forty) than the increase in
edges (a factor slightly less than five). But recall that these additional edges
have an important property. As we know that $G$ is a $k$-faithful CPDAG, both $G$ and $D$ contain the same zero- and
first-order CIs. Therefore, all additional edges in $G$ lead to \emph{no}
further dependencies of order zero or one. It is reasonable to assume
that these additional edges are, thus, part of a disproportionate number of
v-structures as they do not create new paths and thereby new
dependencies. 
\section{Discussion} 
\label{disc_sec}
This paper has investigated the problems of determining
how, for a given set of CI statements of order up to $k$, 
\emph{all} DAGs $k$-faithful to the set can be represented 
and how such a representation can be computed.
We solve both problems showing that such faithful DAGs can be represented
in a compact way as a CPDAG $G$ and then 
proving that the representation $G$ can be computed efficiently. 

The experimental results show 
that, for small values of $k$, this graphical representation 
is also useful as a good estimator 
of the underlying true causal structure in case
of sparse models.
It is considerably better than the $k$-partial graph because further
edges are removed due to the concept of incompatible nodes which allows us
to infer the existence of higher-order independencies.
An additional advantage over $k$-partial graphs  is that
we also obtain edge orientations and can, through this, recover a large portion
of the v-structures in the true DAG.

Our experiments are conducted in the oracle model 
where we assume all CI statements up to order $k$ are known. 
This has the reason that, in this model, we are able to estimate 
best how many incompatible edges are removed.
In future work, it would be interesting to analyze 
how the proposed algorithm performs if 
one would use statistical tests to find the independence statements.
Another interesting topic for future research is to 
extend our algorithmic technique to compute the $k$-faithful
representation, or a good approximation of it, by asking 
conditional independence queries in such a way that the 
number of queries is significantly smaller than the number of 
all CI statements of order up to $k$. This would be 
interesting both in the oracle model and 
when using statistical tests to estimate independencies.  

\section*{Acknowledgments}
This work was supported by the Deutsche Forschungsgemeinschaft (DFG) grant LI 634/4-2.

\bibliography{main}
\bibliographystyle{aaai}

\vspace*{8mm}
\appendix

\centerline{\bf \Large Appendix}

\section{Proof of Correctness of Algorithm~\ref{alg_gen_pdag}}
In this section we rigorously prove Theorem~\ref{thm_alg_correct}
which states that the graph resulting from Algorithm~\ref{alg_gen_pdag} is the representation of the set of $k$-faithful DAGs
to $\mathcal{I}_V^k$. This theorem is the main result of the paper.  
During the proof we, moreover, obtain further results. A few of them were already stated in the main paper.
In particular, as stated below in more detail, Proposition~\ref{prop:represantation} follows
immediately from Corollary~\ref{me_ce_cor} and
Theorem~\ref{thm_alg_correct}, and Proposition~\ref{skel_prop} follows
from Proposition~\ref{same_skel_lemma}. 

In the proof, we will at first consider the PDAG obtained 
from Algorithm~\ref{alg_gen_pdag} at line~\ref{done_line} (after the
for loop) before applying the Meek rules. Throughout this whole
section we will refer to this PDAG as $G_{\text{ep}}$ while we will refer
to the output graph of Algorithm~\ref{alg_gen_pdag} as $G$.
Considering $G_{\text{ep}}$ instead of $G$ will simplify some proofs and we will show the
correctness of the three Meek rules afterwards.

Before we begin, we prove the following lemma. It can be viewed as a stronger version
of Lemma~\ref{early_dir_thm} in the main paper. 

\begin{lemma}
  \label{early_cp_thm}
  Given a set of CIs $\mathcal{I}_V^k$. If we have $(u
  \indep b \: | \: Z)_{\mathcal{I}_V^k}$, $(u \nindep a \:
  | \: Z)_{\mathcal{I}_V^k}$, $(a \nindep b \: | \: Z)_{\mathcal{I}_V^k}$ and $a \not\in Z$, then no DAG $k$-faithful to
  $\mathcal{I}_V^k$ contains a causal path from $a$ to $b$.
\end{lemma}

\begin{proof}
  Assume, for the sake of contradiction, that there is a $k$-faithful DAG
  $D$ which contains a causal path from $a$ to $b$ even though the
  stated conditions hold.
  It follows from $(u \nindep a \: | \: Z)_{\mathcal{I}_V^k}$ and the
  $k$-faithfulness of $D$ that there is a path from $u$ to $a$ which is not blocked by
  $Z$. But as $u$ is supposed to be independent of $b$ given
  $Z$, some node on the causal path from $a$ to $b$ (we call this node $n$) has to be in
  $Z$ blocking this path. Moreover, we know that there is a path from $a$
  to $b$ not blocked by $Z$ as $(a \nindep b \: | \:
  Z)_{\mathcal{I}_V^k}$ has to hold. Thus, there has to be a collider at
  node $a$ (because of $a \not\in Z$) blocking this possible path from
  $u$ to $b$. But this collider would be unblocked by node
  $n$ as it is a successor of $a$ and in $Z$. It follows that $(u \indep b \; | \;
  Z)_{\mathcal{I}_V^k}$ does not hold. A contradiction.
\end{proof}

The following statement follows directly (as
Corollary~\ref{no_pres_cor} did):
\begin{cor}
Assume $\mathcal{I}_V^k$ is a set of CIs. If the nodes $a$ and $b$ are incompatible, no DAG
$k$-faithful to $\mathcal{I}_V^k$ contains a causal path between $a$ and $b$.
\end{cor}

We begin the proof that $G$ is the representation with the statement that every DAG $k$-faithful to a set of independencies
$\mathcal{I}_V^k$ is a subgraph of $G_{\mathrm{ep}}$.
\begin{lemma}
  \label{thm_gen_sg}
  $G_{\text{ep}}$ contains every DAG in $\mathcal{F}(\mathcal{I}_V^k)$.
\end{lemma}

\begin{proof}
    Every $k$-faithful DAG is contained in the graph formed in
  line~\ref{alg_gen_start_line} of Algorithm~\ref{alg_gen_pdag}. Moreover, for every edge $a
  \rightarrow b$  removed in line~\ref{alg_gen_dir_line} the following holds:
  \begin{align*}
    &\exists Z \subseteq V \text{ with } |Z| \leq k \quad \exists u \in V \\
    &[(u \indep b \: | \: Z)_{\mathcal{I}_V^{k}}, \: (u \nindep a \: |
    \: Z)_{\mathcal{I}_V^{k}}, \: (a \nindep b \: | \:
    Z)_{\mathcal{I}_V^{k}}, \: a \not\in Z]
  \end{align*}
  By Lemma~\ref{early_dir_thm} the edge $a \rightarrow b$ is not part
  of any $k$-faithful DAG. 
\end{proof}

Our goal is to show that the edge maximal DAGs have the same skeleton
as $G_{\text{ep}}$.  We have already shown with Lemma~\ref{thm_gen_sg}
that all pairs of nodes, which are adjacent in an edge maximal $k$-faithful
DAG, are also adjacent in $G_{\mathrm{ep}}$ because every $k$-faithful DAG
is contained in $G_{\mathrm{ep}}$. It remains to show that all pairs of
nodes which are adjacent in $G_{\mathrm{ep}}$ are also adjacent in an
edge maximal DAG. 

We begin with the following technical lemma which is necessary for the proof of
Proposition~\ref{add_edge_thm} below. 

\begin{lemma}
  \label{right_ending_lemma}
  Given a DAG $D \in \mathcal{F}(\mathcal{I}_V^k)$, two nodes $a,b \in
  V$ and a set $Z$ with $|Z| \leq k$ such that the following holds:
  $(a \nindep b \: | \: Z)_{\mathcal{I}_V^k}$, $(a \nindep b \: |
  \: Z \backslash \mathrm{De}(b))_{\mathcal{I}_V^k}$ and $a \not \in \mathrm{De}(b)$. Then there
  is a path d-connecting $a$ and $b$ in $D$ given $Z$ which ends with
  $\rightarrow b$. 
\end{lemma}

\begin{proof}
  Assume, for the sake of contradiction, that there is no path
  d-connecting $a$ and $b$ given $Z$ ending with $\rightarrow b$ in $D$.
  This means every path ends with the edge $\leftarrow b$. Moreover, we know that there
  cannot be a causal path from $b$ to $a$ because $a \not\in
  \mathrm{De}(b)$.
  Then, it is clear that every path $p$ which
  d-connects $a$ and $b$ given $Z$ in $D$ contains at least one
  collider. We also note that on $p$ every node unblocking
  the collider closest to $b$ is a descendant of $b$. We will now
  consider the set $Z' = Z \: \backslash \: \mathrm{De}(b)$ meaning we
  remove all nodes from $Z$ which are a descendant of $b$. We will
  show that there can be no path d-connecting $a$ and $b$ given
  $Z'$. This will contradict the assumption that $(a \nindep b \:
  | \: Z')_{\mathcal{I}_V^k}$ holds for $Z' = Z \: \backslash \: \mathrm{De}(b)$.

  Every path d-connecting $a$ and $b$ given $Z'$ contains a node $x
  \in \mathrm{De}(b)$. If this were not the case and there actually is
  such a path which contains no node in $\mathrm{De}(b)$, then
  this path would d-connect $a$ and $b$ given $Z$ as well. Moreover,
  this path would have to end with the edge $\rightarrow b$ (else the
  node preceding $b$ on the path is a descendant of $b$). 
  But we have assumed above, for the sake of contradiction, that there
  is no path d-connecting $a$ and $b$ given $Z$ in $D$ ending with
  $\rightarrow b$. 
  \begin{figure}
    \centering
    \begin{tikzpicture}
      \node (a1) at (0, 0) {$a$};
      \node (dots11) at (1, 0) {$\dots$}
      edge [->] (a1);
      \node (c1) at (2,0) {$c_1$}
      edge [<-] (dots11);
      \node (bdots1) at (2, -1) {$\vdots$}
      edge [<-] (c1);
      \node (d1) at (2, -2) {$d_1$}
      edge [<-] (bdots1);
      \node (dots12) at (3, 0) {$\dots$}
      edge [->] (c1);
      \node (x1) at (4, 0) {$x$}
      edge [->] (dots12)
      edge [->, dotted, bend right] (a1);
      \node (dots13) at (5,0) {$\dots$}
      edge [->] (x1);
      \node (dots14) at (5, 0.5) {$\dots$}
      edge [->, bend right] (x1);
      \node (b1) at (6,0) {$b$}
      edge [<-] (dots13)
      edge [->, bend right] (dots14);
      \node (l1) at (0, -1) {(a)};
      \node (a2) at (0, -4) {$a$};
      \node (dots21) at (1, -4) {$\dots$}
      edge [->] (a2);
      \node (x2) at (2, -4) {$x$}
      edge [<-] (dots21);
      \node (dots22) at (3, -4) {$\dots$}
      edge [<-] (x2);
      \node (c2) at (4,-4) {$c_2$}
      edge [<-] (dots22);
      \node (bdots2) at (4, -5) {$\vdots$}
      edge [<-] (c2);
      \node (d2) at (4, -6) {$d_2$}
      edge [<-] (bdots2);
      \node (dots23) at (5,-4) {$\dots$}
      edge [->] (c2);
      \node (dots24) at (4, -2.75) {$\dots$}
      edge [->, bend right] (x2);
      \node (b1) at (6,-4) {$b$}
      edge [<-] (dots23)
      edge [<-, dotted, bend right] (x2)
      edge [->, bend right] (dots24);
      \node (l2) at (0, -5) {(b)};
    \end{tikzpicture}
    \caption{The two cases considered in the proof of
      Lemma~\ref{right_ending_lemma}. In (a) there is an edge
      $\leftarrow x$ on the path between $a$ and $b$. A causal path
      from $x$ to $a$ (dotted line) is impossible because then there
      would be a causal path from $b$ to $a$. We show that the
      collider $c_1$ is unblocked. In (b) the edge $x \rightarrow$ is
      part of the path between $a$ and $b$. A causal path from $x$ to
      $b$ (dotted line) is impossible because this would imply a
      cycle. We show that the collider $c_2$ is unblocked.}
    \label{cc_lemma_cases}
  \end{figure}
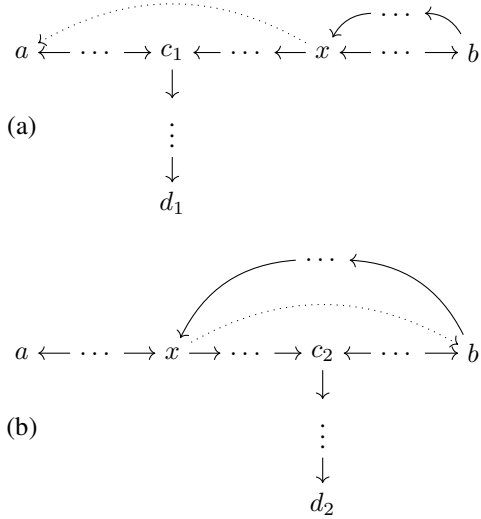
  
  We will consider a path $p'$ d-connecting $a$ and $b$ given $Z'$
  which contains a node $x \in \mathrm{De}(b)$. This node cannot be a
  collider $\rightarrow x \leftarrow$ in $p'$, because $x$ is not in
  $Z'$ and neither is any descendant $y$ of $x$, as $y$ is by
  transitivity a descendant of $b$ as well. Thus, the collider
  $\rightarrow x \leftarrow$ would be unblocked. It follows
  that in $p'$ there is an edge $\leftarrow x$ or an edge $x
  \rightarrow$. We investigate these two cases which are also
  displayed in Figure~\ref{cc_lemma_cases}: 
  \begin{enumerate}
  \item
    Consider the edge $\leftarrow x$ is in $p'$. This case is shown in
    part (a) of Figure~\ref{cc_lemma_cases}. We denote the subpath
    between $a$ and $x$ of $p'$ as $p'_{a-x}$. This subpath cannot be
    causal from $x$ to $a$ as then there would be a causal path from
    $b$ to $a$ because $x$ is a descendant of $b$. But we required
    that $a \not\in \mathrm{De}(b)$ holds.

    This means that there has to be a collider on $p'_{a-x}$. We
    will look at the collider $c_1$ closest to $x$. The collider $c_1$,
    however, cannot be unblocked by a node $d_1$ in $Z'$. This is because
    $d_1$ would be a descendant of $b$.
  \item
    Consider the edge $x \rightarrow$ is in $p'$. This case is shown in
    part (b) of Figure~\ref{cc_lemma_cases}. We denote the subpath
    between $x$ and $b$ of $p'$ as $p'_{x-b}$. This subpath cannot
    be causal from $x$ to $b$ as then there would be a cycle because
    $x$ is a descendant of $b$.

    It follows that there is a collider on the subpath $p'_{x-b}$. We
    look at the collider $c_2$ closest to $x$. This collider cannot be
    unblocked by a node $d_2$ in $Z'$ because $d_2$ would be a
    descendant of $b$.
  \end{enumerate}
  We have seen that there cannot be a path d-connecting $a$ and $b$
  given $Z'$ in $D$. This is a contradiction to the requirement that  $(a \nindep b \:
  | \: Z')_{\mathcal{I}_V^k}$ holds for $Z' =  Z \: \backslash \: \mathrm{De}(b)$. Therefore, we conclude that indeed there is a path d-connecting
  $a$ and $b$ given $Z$ ending with $\rightarrow b$ in $D$.
\end{proof}

The following lemma is of central importance for this
section. We show that every edge in $G_{\text{ep}}$ can be added to a
$k$-faithful DAG $D$ iff this does not produce a cycle. This is an
important step towards showing that the edge maximal DAGs $k$-faithful to
$\mathcal{I}_V^k$ have the same skeleton as $G_{\text{ep}}$. 

\begin{lemma}
  \label{add_edge_thm}
  Given a DAG $D \in \mathcal{F}(\mathcal{I}_V^{k})$ and $a,b \in V$ nonadjacent in
  $D$. The DAG $D ' = D \cup \{a \rightarrow b\}$ is $k$-faithful to
  $\mathcal{I}_V^k$ iff $a \not\in \mathrm{De}_D(b)$ and $a \rightarrow b \in
  G_{\mathrm{ep}}$ hold. 
\end{lemma}

\begin{proof}
  We show two directions.
  We begin by showing that if the DAG $D' = D \cup \{a \rightarrow b\}$ is
  $k$-faithful to $\mathcal{I}_V^k$, then $a  \not\in \mathrm{De}_D(b)$ and $a \rightarrow b \in
  G_{\mathrm{ep}}$ hold. Clearly, $a$ cannot be in $\mathrm{De}_D(b)$ as then
  there would be a cycle in $D'$. Moreover,
  every $k$-faithful DAG is contained in $G_{\mathrm{ep}}$
  (Lemma~\ref{thm_gen_sg}) and because $D'$ is $k$-faithful it follows that $a \rightarrow
  b$ is in $G_{\mathrm{ep}}$.

  We will now show the more interesting direction that if $a \not\in \mathrm{De}_D(b)$ and $a \rightarrow b \in
  G_{\mathrm{ep}}$ are satisfied, the DAG $D' = D \cup \{a \rightarrow b\}$ is $k$-faithful to
  $\mathcal{I}_V^k$. We prove this by showing that the following
  holds:
  \begin{align*}
    &\forall Z \subseteq V \text{ with } |Z| \leq k \quad \forall u,v \in V \\
    &[(u \indep v \: | \:
    Z)_{D'} \iff (u \indep v \: | \: Z)_{D}]
  \end{align*}
  We show two directions. We begin with the direction $(u \indep v \: | \: Z)_{D'} \implies (u \indep v \: | \: Z)_{D}$. 
  Every conditional independence of order $\leq k$ in $D'$ is also in $D$ because $D$ is
  a subgraph of $D'$.
  The second direction $(u \nindep v \: | \: Z)_{D'} \implies (u \nindep v \: | \: Z)_{D}$ is more intricate.
  We will prove that every conditional dependence of order $\leq k$ in $D'$
  is also in $D$ by considering a path $p'$ which d-connects $u$ and $v$ given a set $Z$ in
  $D'$. Then we show that there will also be a path $p$ in $D$ not blocked
  by $Z$.

  There are two cases to consider displayed in Figure~\ref{cases_add_tikz}.
  \begin{figure}
    \centering
    \begin{tikzpicture}
      \node (v1) at (6.4, 1) {$v$};
      \node (rinv1) at (5.5, 1) {$\dots$}
      edge[-] (v1);
      \node (b1) at (4.6, 1) {$b$}
      edge[-] (rinv1);
      \node (a1) at (3.6, 1) {$a$}
      edge[->] (b1);
      \node (linv1) at (2.7, 1) {$\dots$}
      edge[->] (a1);
      \node (d1) at (1.8, 1) {$d$}
      edge[->] (linv1)
      edge[->, dashed, bend left, red] (b1);
      \node (llinv1) at (0.9, 1) {$\dots$}
      edge[<-] (d1);
      \node (u1) at (0, 1) {$u$}
      edge[-] (llinv1);
      \node (l1) at (-0.3, 3.1) {(a)};
      \draw (-0.3, 0.7) -- (-0.3, 0.6) -- (6.7, 0.6) -- (6.7,
      0.7);
      \node (pd1) at (3.2, 0.3) {$p'$};
      \draw[red] (-0.3, 1.9) -- (-0.3, 2.3) -- (4.6, 2.3) -- (4.6,
      1.9);
      \draw[red] (1.8, 2.3) -- (1.8, 2);
      \node[red] (wub1) at (2.2, 2.5) {$w_{u-b}$};
      \node[red] (qdb1) at (3.2, 2.1) {$q_{d-b}$};
      \node[red] (pdud1) at (0.8, 2.1) {$p'_{u-d}$};
      \draw[red] (-0.3, 2.3) -- (-0.3, 2.7) -- (6.7, 2.7) -- (6.7,
      2.3);
      \draw[red] (4.6, 2.7) -- (4.6, 2.3);
      \node[red] (pdbv1) at (5.6, 2.5) {$p'_{b-v}$};
      \node[red] (p1) at (3.6, 2.9) {$w$};

      \node (v2) at (7.2, -6) {$v$};
      \node (rrinv2) at (6.3, -6) {$\dots$}
      edge[-] (v2);
      \node (dd2) at (5.4, -6) {$d'$}
      edge[->] (rrinv2);
      \node (rinv2) at (4.5, -6) {$\dots$}
      edge[<-] (dd2);
      \node (c2) at (3.6,-6) {$c$}
      edge[<-] (rinv2);
      \node (linv2) at (2.7, -6) {$\dots$}
      edge[->] (c2);
      \node (d2) at (1.8, -6) {$d$}
      edge[->] (linv2);
      \node (llinv2) at (0.9, -6) {$\dots$}
      edge[<-] (d2);
      \node (u2) at (0, -6) {$u$}
      edge[-] (llinv2);
      \node (binv2) at (3.6, -5.2) {$\vdots$}
      edge[<-] (c2);
      \node (a2) at (3.6, -4.4) {$a$}
      edge[<-] (binv2);
      \node (b2) at (3.6, -3.6) {$b$}
      edge[<-] (a2)
      edge[<-, dashed, bend left, red] (dd2)
      edge[<-, dashed, bend right, red] (d2);
      \node (bbinv2) at (3.6, -2.8)  {$\vdots$}
      edge[<-] (b2);
      \node (x2) at (3.6, -2) {$x$}
      edge[<-] (bbinv2);
      \node (l2) at (-0.3, -0.6) {(b)};
      \draw (-0.3, -6.3) -- (-0.3, -6.4) -- (7.5, -6.4) -- (7.5,
      -6.3);
      \node (pd2) at (3.9, -6.7) {$p'$};
      \draw[red] (-0.3, -1.8) -- (-0.3, -1.4) -- (3.6, -1.4);
      \draw[red] (1.8, -1.8) -- (1.8, -1.4);
      \draw[red] (3.6, -1.8) -- (3.6, -1.4) -- (7.5, -1.4) -- (7.5,
      -1.8);
      \draw[red] (5.4, -1.8) -- (5.4, -1.4);
      \node[red] (wub2) at (1.6, -1.2) {$w_{u-b}$};
      \node[red] (wvb2) at (5.6, -1.2) {$w_{v-b}$};
      \node[red] (qdb) at (2.7, -1.6) {$q_{d-b}$};
      \node[red] (qddb) at (4.6, -1.6) {$q_{d'-b}$};
      \draw[red] (-0.3, -1.4) -- (-0.3, -1) -- (7.5, -1) -- (7.5,
      -1.4);
      \node[red] (p2) at (3.6, -0.8) {$w$};
      \node[red] (pdud2) at (0.8, -1.6) {$p'_{u-d}$};
      \node[red] (pdvdd2) at (6.5, -1.6) {$p'_{v-d}$};
      \draw[red] (3.6, -1) -- (3.6, -1.4);
    \end{tikzpicture}
    \caption{The two cases in which the edge $a \rightarrow b$ can be part
      of a path $p'$ d-connecting $u$ and $v$ given $Z$ in $D'$. In (a) the edge $a
      \rightarrow b$ is on the path $p'$. In (b) the edge $a \rightarrow
      b$ is part of a chain which unblocks a collider $c$ on the path
      $p'$. Moreover, it is indicated in red how a way $w$ connecting
      $u$ and $v$ in $D$ given $Z$ looks like. It is a concatenation
      of subpaths of $p'$ with the new path $q_{d-b}$ (and $q_{d'-b}$
      in part (b)) which does not
      use the edge $a \rightarrow b$ (as this edge is only in $D'$ and
      not in $D$). It is therefore vital to show that the path
      $q_{d-b}$ exists. Note that for the concenation to work we make
      sure that $q_{d-b}$ ends with $\rightarrow b$ and that $d$ is no
      collider. }

    \label{cases_add_tikz}
  \end{figure}
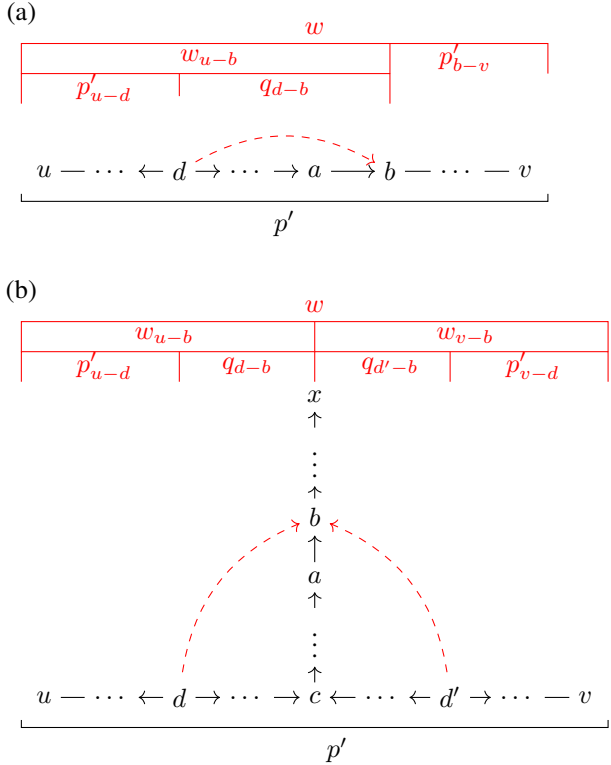
  The case (a) describes the situation when the edge $a \rightarrow b$
  is on the path $p'$ d-connecting $u$ and $v$ given $Z$ in $D'$.
  The case (b) appears when a collider $c$ on the path $p'$ is unblocked by the descendant
  $x$ which is in $Z$ and the edge $a \rightarrow b$ is on the causal
  path from $c$ to $x$. The nodes $d$ and $d'$ as well as the red
  arrows and boxes
  visualize later parts of the proof and can be ignored for now.
  It is clear that any further occurrence
  of the edge $a \rightarrow b$ in $p'$ would be redundant. 
  Moreover, it is obvious that if none of the two cases applies and the edge is
  neither present in $p'$ nor takes part in unblocking a collider, the
  same path $p'$ will also exist in $D$. 

  We prove that for the two cases in Figure~\ref{cases_add_tikz}
  there is a path $p$ d-connecting $u$ and $v$ given $Z$ in $D$.
  We do this by showing that a way $w_{u-b}$ connecting $u$ and $b$
  given $Z$ exists which does not
  contain the edge $a \rightarrow b$, but is still ending with
  $\rightarrow b$. We
  will first argue that if such a way $w_{u-b}$ exists, then there will be a path $p$ d-connecting $u$ and $v$ given
  $Z$ in $D$. Afterwards, we prove the existence of $w_{u-b}$. 
  
  In both cases illustrated in Figure~\ref{cases_add_tikz} we can construct the desired
  path $p$. In case (a) we have a way $w$ which is the concatenation of the way
  $w_{u-b}$ and the path $p'_{b-v}$ which is the subpath between $b$
  and $v$ of path $p'$. The concatenation of $w_{u-b}$ and $p'_{b-v}$
  is valid because $w_{u-b}$ ends with $\rightarrow b$ just as $p'_{u-b}$
  did and because of the fact that $p'$ is a valid path. If $b$ is no
  collider in $p'$, it will also not be a collider in $w$ and, vice
  versa, if it is an unblocked collider in $p'$, then it is also an
  unblocked collider in $w$.
  In case (b) the way $w$ is the concatenation
  of the way $w_{u-b}$ and $w_{v-b}$ (we see below that this way
  exists as well by
  symmetry) as the node $x$ unblocks the collider at node $b$. 
  Finally, we know that the existence of a way which d-connects $u$ and $v$ given
  $Z$ implies the existence of a path with the same property. This
  means we are able to obtain the desired path $p$ in both cases and it follows that
  $(u \nindep v \: | \: Z)_D$ holds.
  
  Thus, it remains to find a way $w_{u-b}$ d-connecting $u$ and $b$
  given $Z$ in $D$ which ends with $\rightarrow b$ under the
  assumption that there is a path $p'_{u-b}: u = v_1, v_2, \dots ,
  v_{l-1} = a, v_l = b$ of length $l$ which d-connects $u$ and
  $b$ given $Z$ in $D'$ and ends with the edge $a \rightarrow b$.
  Let
  $d = v_i$ be the node with the minimal $i$ such that $b \in
  \mathrm{De}_{D'}(v_i)$ holds. Then either $d = u$ or we have $v_{i-1}
  \leftarrow d$ on $p'_{u-b}$. We will use this fact below to argue
  that there can be no collider at node $d$. Moreover, $v_j \not\in Z$ holds
  for $i \leq j < l$, because the path from $d = v_i$ to $b$ is causal
  and we assumed that $p'_{u-b}$ is a valid path d-connecting $u$ and $b$
  given $Z$. In particular, $a \not\in Z$ follows as $i \leq
  l-1$ holds for $d$.
  We will show that there is a path $q_{d-b}$
  d-connecting $d$ and $b$ in $D$ which ends with $\rightarrow
  b$. Concatenating the subpath $p'_{u-d}$ with this path
  $q_{d-b}$ will result in the required way $w_{u-b}$ because there
  can be no collider at node $d$ and $d \not\in Z$ holds. 

  The path $q_{d-b}$ exists in $D$ because of the following argument: The
  node $d$ cannot be a descendant of $b$ because then with $a$
  being a descendant of $d$, it would follow that $a$ is a descendant
  of $b$ contradicting our assumption that $a \not\in \text{De}_D(b)$.  Moreover every node $v_j$ with $j \geq i$ is
  not in $Z$ as seen above. Then the statement $(d \nindep b \: | \: Z')_D$
  holds for every subset $Z'$ of $Z$
  because $(d \indep b \: | \: Z')_D$ would imply the following
  contradiction: We know that $(a \nindep b \: | \: Z')_D$ (this holds
  for every $Z'$ with $|Z'| \leq k$ because the edge $a \rightarrow b$
  is in $G$) and also $(d \nindep a \: | \: Z')_D$ hold (because of the fact that no node $v_j$ with $j \geq i$ is
  in $Z$ meaning the same follows for $Z'$ as it is a subset and
  because there is a causal path from $d$ to $a$ by definition of $d$).
  Note that because of the $k$-faithfulness of $D$, the
  same statements hold according to $\mathcal{I}_V^k$ as well. 
  With $a \not\in Z$ (and therefore also $a \not \in Z'$) the edge $a \rightarrow b$
  would have been removed from $G_{\mathrm{ep}}$ because
  these are exactly the conditions checked in
  line~\ref{cond_gen_line} of Algorithm~\ref{alg_gen_pdag}.
  However, this would mean that we are not able to add the edge $a \rightarrow b$
  to $D'$. A contradiction. This means that $(d \nindep b \: | \:
  Z')_D$ holds for every subset $Z'$ of $Z$ and therefore in
  particular for $Z' = Z \backslash \mathrm{De}(b)$. 
   With Lemma~\ref{right_ending_lemma} it
  follows that there is a path $q_{d-b}$ d-connecting $d$ and $b$
  given $Z$ ending with $\rightarrow b$.  
\end{proof}

In addition to Lemma~\ref{add_edge_thm}, we need the following Lemma for the proof of Theorem~\ref{same_skel_lemma} below. 
\begin{lemma}
  \label{can_add_case}
  If $a \leftarrow b \in G_{\mathrm{ep}}$
  and $a \rightarrow b \not\in G_{\mathrm{ep}}$, it follows that $b \not\in
  \mathrm{De}_D(a)$ holds for every DAG $D \in \mathcal{F}(\mathcal{I}_V^k)$.
\end{lemma}

\begin{proof}
  Having the edge $a \leftarrow b$ in $G_{\mathrm{ep}}$ but not the edge $a
  \rightarrow b$ implies that the following holds:
  \begin{align*}
    &\exists Z \subseteq V \text{ with } |Z| \leq k \quad \exists u \in V \\
    &[(u \indep b \: | \: Z)_{\mathcal{I}_V^{k}}, \: (u \nindep a \: |
    \: Z)_{\mathcal{I}_V^{k}}, \: (a \nindep b \: | \:
    Z)_{\mathcal{I}_V^{k}}, \: a \not\in Z]
  \end{align*}
  This is because these are exactly the conditions required to remove the
  edge $a \rightarrow b$ in line~\ref{alg_gen_dir_line} of
  algorithm~\ref{alg_gen_pdag}. From Lemma~\ref{early_cp_thm} we
  know that these conditions mean that no $k$-faithful DAG contains a
  causal path from $a$ to $b$. Thus, $b \not\in \mathrm{De}_D(a)$ holds. 
\end{proof}

We obtain one of the main results of this section that the edge
maximal $k$-faithful DAGs have same skeleton as $G_{\mathrm{ep}}$. 
\begin{prop}
  \label{same_skel_lemma}
  The edge maximal DAGs $k$-faithful to $\mathcal{I}_V^k$ have the same
  skeleton as $G_{\mathrm{ep}}$. 
\end{prop}

\begin{proof}
  We show two directions.
  If $a$ and $b$ are adjacent in an edge maximal $k$-faithful DAG
  $D$, they are also adjacent in $G_{\mathrm{ep}}$. This follows from
  Lemma~\ref{thm_gen_sg} because every $k$-faithful DAG is contained in
  $G_{\mathrm{ep}}$.
  
  The second direction is more intricate. We show that if $a$ and $b$ are adjacent in $G_{\mathrm{ep}}$, they are also adjacent
  in any edge maximal $k$-faithful DAG $D$. Assume, for the sake of contradiction, that $a$ and $b$
  are not adjacent in an edge maximal $k$-faithful DAG $D$. We consider three cases:
  \begin{enumerate}
    \item The edges $a \rightarrow b$ and $a \leftarrow b$ are in
      $G_{\mathrm{ep}}$. From Lemma~\ref{add_edge_thm} we know that the edge $a
      \rightarrow b$ can be added to $D$ if $a \not\in
      \text{De}_D(b)$. If on the other hand $a \in \text{De}_D(b)$ holds, then the
      edge $a \leftarrow b$ can be added, because in this case $b \not\in
      \text{De}_D(a)$ has to hold (else there would be a cycle in $D$). Thus,
      $D$ is not edge maximal. A contradiction.
    \item \label{cases_same_skel_sec}
      The edge $a \leftarrow b$ is in $G_{\mathrm{ep}}$ and the edge $a
      \rightarrow b$ is not. From Lemma~\ref{can_add_case} it follows that $b
      \not\in \text{De}_D(a)$ holds for every $k$-faithful DAG. Thus, as
      shown in Lemma~\ref{add_edge_thm} the edge $a \leftarrow b$ can be added to $D$. This means
      that $D$ is not edge maximal. A contradiction.
    \item The edge $a \rightarrow b$ is in $G_{\mathrm{ep}}$ and the edge $a
      \leftarrow b$ is not. This case is symmetrical to case~\ref{cases_same_skel_sec}
      above. 
  \end{enumerate}
\end{proof}

From this, we can immediately conclude Proposition~\ref{skel_prop} given in the main paper.
This is due to the fact that the two points stated there are the exact
reasons two nodes are nonadjacent in $G_{\mathrm{ep}}$ and by Proposition~\ref{same_skel_lemma} in the edge maximal $k$-faithful DAGs. These all have
the same skeleton which is precisely the skeleton of the
representation. 

Of all DAGs $k$-faithful to $\mathcal{I}_V^k$ the edge maximal
DAGs possess another very unique property. We will prove that these DAGs form a Markov equivalence
class. This result has far reaching
consequences. In order to show this, we state that the edge maximal
$k$-faithful DAGs not only have the same skeleton, but also the same set of v-structures as $G_{\mathrm{ep}}$: 

\begin{prop}
  \label{em_vs_thm}
  For all $a, b, c \in V$ it is true: $a \rightarrow c \leftarrow b$ is a v-structure in an edge maximal DAG $D \in \mathcal{F}(\mathcal{I}_V^k)$ iff $a \rightarrow c \leftarrow b$ is a v-structure in $G_{\mathrm{ep}}$.
\end{prop}

\begin{proof}
  We begin by showing that if a v-structure $a \rightarrow c \leftarrow b$ is in $D$, will also be in $G_{\mathrm{ep}}$.
  We note that, because $a$ and $c$ as well as
  $c$ and $b$ are adjacent in the $k$-faithful DAG $D$, the following holds:
  \[
    \forall Z \text{ with } |Z| \leq k \quad (a \nindep c \: | \: Z)_{\mathcal{I}_V^k} \text{
      and } (c \nindep b \: | \: Z)_{\mathcal{I}_V^k}
  \]
  The nodes $a$ and $b$ are not adjacent in $D$ and as $D$ is edge
  maximal, it follows from the fact that $G_{\mathrm{ep}}$ and $D$ have the same
  skeleton (Proposition~\ref{same_skel_lemma}) that they will not be adjacent in
  $G_{\mathrm{ep}}$ either. We will now show that the edges $a \leftarrow c$
  and $c \rightarrow b$ are not in $G_{\mathrm{ep}}$. Then we can conclude from the fact that
  every $k$-faithful DAG is contained in $G_{\mathrm{ep}}$
  (Lemma~\ref{thm_gen_sg}) that the v-structure $a \rightarrow c
  \leftarrow b$ is in $G_{\mathrm{ep}}$.
  
  If the nodes $a$ and $b$ are not adjacent in $G_{\mathrm{ep}}$ there are two
  possible reasons for this:
  \begin{enumerate}
  \item
    The edge $a - b$ was not added to $G_{\mathrm{ep}}$ in
    line~\ref{alg_gen_start_line} because of an independence $(a
    \indep b \: | \:
    Z)_{\mathcal{I}_V^k}$. Moreover, because we have $a \rightarrow c
    \leftarrow b$ in the $k$-faithful DAG $D$, it follows that $c
    \not\in Z$ has to hold. This means that the edges are
    directed $a \rightarrow c \leftarrow b$ in $G_{\mathrm{ep}}$ because the conditions in line~\ref{cond_gen_line} of
    Algorithm~\ref{alg_gen_pdag} are met
    \[
      (a \indep b \: | \: Z)_{\mathcal{I}_V^k}, \: (a \nindep c \: | \:
      Z)_{\mathcal{I}_V^k}, \: (c \nindep b \: | \:
      Z)_{\mathcal{I}_V^k}, \: c \not\in Z
    \]
    and therefore the edges $a \leftarrow c$ and $c \rightarrow b$ were removed
    from $G_{\mathrm{ep}}$.
  \item
    \label{sec_vs_cases}
    The edges $a \rightarrow b$ and $a \leftarrow b$ were removed in
    line~\ref{alg_gen_dir_line}. This case is displayed in
    Figure~\ref{same_vs_cases_tikz}.
    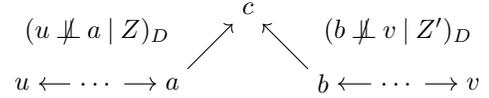
\begin{figure}
      \centering
      \begin{tikzpicture}
        \node (u) at (0,0) {$u$};
        \node (dots1) at (1,0) {$\dots$}
        edge [->] (u);
        \node (l1) at (1,0.7) {$(u \nindep a \: | \: Z)_{D}$};
        \node (a) at (2,0) {$a$}
        edge [<-] (dots1);
        \node (c) at (3,1) {$c$}
        edge [<-] (a);
        \node (b) at (4,0) {$b$}
        edge [->] (c);
        \node (dots2) at (5,0) {$\dots$}
        edge [->] (b);
        \node (l2) at (5, 0.7) {$(b \nindep v \: | \: Z')_{D}$};
        \node (v) at (6,0) {$v$}
        edge [<-] (dots2);
      \end{tikzpicture}
      \caption{Case~\ref{sec_vs_cases} of the proof of
        Proposition~\ref{em_vs_thm}. The v-structure $a \rightarrow c
        \leftarrow b$ is in $D$ and there exist $Z$ and $Z'$ such that
        $(u \nindep a \: | \: Z)_D$ and $(b \nindep v \: | \: Z')_D$ hold. We argue that $(u \nindep c \: | \:
        Z)_{D}$, $(v \nindep c \: | \: Z')_D$, $c \not\in Z$ and $c
        \not\in Z'$ hold as well. }
      \label{same_vs_cases_tikz}
    \end{figure}

    This means we have nodes $u$ and $v$ and sets $Z$ and $Z'$ such that
    \[
      (u \indep b \: | \: Z)_{\mathcal{I}_V^k}  \land (u \nindep a \: | \: Z)_{\mathcal{I}_V^k} \land (a
      \nindep b \: | \: Z)_{\mathcal{I}_V^k} \land a \not\in Z
    \]
    and
    \[
      (v \indep a \: | \: Z')_{\mathcal{I}_V^k} \land  (v \nindep b \: | \: Z')_{\mathcal{I}_V^k} \land (b \nindep a \: | \: Z')_{\mathcal{I}_V^k} \land b \not\in Z'
    \]
    hold. 
    Then $(u \nindep c \: | \:
    Z)_{\mathcal{I}_V^k}$ and $(v \nindep c \: | \:
    Z')_{\mathcal{I}_V^k}$ hold as well because of $a \not\in Z$ and $b
    \not\in Z'$ and the fact that with the edges $a \rightarrow c$ and
    $b \rightarrow c$ in the $k$-faithful DAG $D$ there is neither a collider at node $a$ nor
    at node $b$. On the other hand there is a collider at node
    $c$ (one the path from $u$ to $b$ as well as from $v$ to $a$)
    and therefore $c \not\in Z$ and $c \not\in Z'$ hold. 
    Then with $(c \nindep b \:
    | \: Z)_{\mathcal{I}_V^k}$ (with the edge $c \leftarrow b$ in $D$ there cannot
    be any independence) and $(u \indep b \: | \:
    Z)_{\mathcal{I}_V^k}$ the edge $c \rightarrow b$ is removed from
    $G$ because the conditions in line~\ref{cond_gen_line} of
      Algorithm~\ref{alg_gen_pdag} are met. The
    edge $a \leftarrow c$ is removed, too, as additional to $(v \nindep c \:
    | \: Z')_{\mathcal{I}_V^k}$ and $c \not\in Z'$ the statements $(a
    \nindep c \: | \: Z')_{\mathcal{I}_V}$ and $(v \indep a \: | \:
    Z')_{\mathcal{I}_V}$ hold. 
    Therefore we have the v-structure $a \rightarrow c \leftarrow
    b$ in $G$ as well.
  \end{enumerate}
  Now we show that a v-structure in $G_{\mathrm{ep}}$ will be present in $D$
  as well. It follows from Lemma~\ref{thm_gen_sg} that if we have
  $a \rightarrow c \leftarrow b$ in $G_{\mathrm{ep}}$, $D$ can neither contain an
  edge between $a$ and $b$ nor the edges $a \leftarrow c$ or $c
  \rightarrow b$. Moreover because of the edge maximality of $D$ and
  the fact that $D$ and $G_{\mathrm{ep}}$ have the same skeleton
  (Proposition~\ref{same_skel_lemma}) that the edges $a \rightarrow c \leftarrow b$
  will be present in $D$.
\end{proof}

We will now include the Meek rules in our argument in order to show the
following important result which shows a way to obtain $k$-faithful DAGs
from the graph $G$ which is the final result of Algorithm~\ref{alg_gen_pdag}. 

\begin{cor}
  \label{me_ce_cor}
  The set of edge maximal DAGs $k$-faithful to $\mathcal{I}_V^k$ is the
  Markov equivalence class formed by all consistent extensions of the
  CPDAG $G$.
\end{cor}

\begin{proof}
  Proposition~\ref{same_skel_lemma} states that the edge maximal DAGs have the same
  skeleton and Proposition~\ref{em_vs_thm} states that they have the same set of
  v-structures as $G_{\mathrm{ep}}$. Thus, these DAGs form a Markov
  equivalence class which is exactly the set of all consistent
  extensions of $G_{\mathrm{ep}}$. It immediately follows that the
  graph $G$ which results from applying the Meek rules to
  $G_{\mathrm{ep}}$ is a CPDAG. Moreover, by correctness of the Meek rules (these
  rules neither create a new v-structure nor a cycle~\cite{Meek})
  $G$ has the same set of consistent extensions as $G_{\mathrm{ep}}$. 
\end{proof}

Finally, it becomes clear why we can apply the Meek
rules to $G_{\mathrm{ep}}$ in Algorithm~\ref{alg_gen_pdag}. As shown by Meek~\shortcite{Meek} these
rules maximally extend a PDAG whose consistent extensions form a
Markov equivalence class into a CPDAG and the edge maximal DAGs are the
Markov equivalence class formed by the consistent extensions of
$G_{\mathrm{ep}}$. If an edge
$a \rightarrow b$ gets directed by one of the Meek rules, this means that it
is in every consistent extension (while the edge $a
\leftarrow b$ is in no consistent extension). That the
application of these three rules is correct for all $k$-faithful DAGs — not only the edge maximal
ones — will be argued in the following proof of the main result of the paper, the correctness of Algorithm~\ref{alg_gen_pdag} (Theorem~\ref{thm_alg_correct}).

\begin{proof}[Proof of Theorem~\ref{thm_alg_correct}]
  A representation of the set $\mathcal{F}(\mathcal{I}_V^k)$ is a
  minimal graph that contains every graph in
  $\mathcal{F}(\mathcal{I}_V^k)$.
  We begin by proving that $G$ indeed contains every DAG $k$-faithful to
  $\mathcal{I}_V^k$. We do this by showing that every $k$-faithful DAG is
  a subgraph of a consistent extension of $G$.
  Consider the $k$-faithful DAG $D \in \mathcal{F}(\mathcal{I}_V^k)$. The DAG $D$ has to be a subgraph of
  some edge maximal $k$-faithful DAG.
  We know from Corollary~\ref{me_ce_cor} that
  every edge maximal DAG is a consistent extension of $G$ . Thus, $D$ is a subgraph of a
  consistent extension of $G$.

  We show now that $G$ is indeed minimal. This holds as deleting or directing an edge in $G$ would
  immediately violate the condition that $G$ contains every $k$-faithful
  DAG. This follows as we know that $G$ is a CPDAG
  representing a Markov equivalence class of $k$-faithful DAGs, namely the edge maximal ones. 
\end{proof}

From the above theorems we can deduce an interesting fact. Lemma~\ref{add_edge_thm} holds
for $G_{\mathrm{ep}}$ and not only for $G$ and the only constraint we
impose on adding edges is that they do not produce a cycle. Thus, if
we have an edge $a - b$ in $G_{\mathrm{ep}}$, an edge $a
\rightarrow b$ in $G$ (meaning the edge $a\rightarrow b$ has been
directed by one of the Meek rules) and a DAG $D \in \mathcal{F}(\mathcal{I}_V^k)$,
it follows that either $a \rightarrow b$ is in $D$ or there is a
causal path from $a$ to $b$.

Finally, we are able to derive Proposition~\ref{prop:represantation}.
This follows immediately from Corollary~\ref{me_ce_cor} and Theorem~\ref{thm_alg_correct}. It shows
that the notion of a representation is a
generalization of the notion of a CPDAG. More precisely, for every $k$
there is a subclass of CPDAGs (let us call these $k$-CPDAGs) which are the
representation of a set of DAGs $k$-faithful to a set
$\mathcal{I}_V^k$ for a fixed $|V| = n$. In particular, the set of $l$-CPDAGs is a subset of
the set of $l+1$-CPDAGs and the set of $n-2$-CPDAGs is the set of all CPDAGs. Further investigations of these structures
might be interesting, for example, for the open question of counting
the number of Markov equivalence classes (which is equal to the number
of CPDAGs) for a given number $n$ of
nodes~\cite{Counting17}. Notably, Textor et al.~\shortcite{Idel15} analyzed the number
of $0$-CPDAGs (they use a different representation termed SMIG).

\section{Additional proofs for the case of $k=0$}
All results in the previous section hold for the setting of marginal
independencies as well. But for this special case, there
already existed an algorithm \cite{Idel15,PearlWermuth94} whose formal
proof of correctness was, to our knowledge, never published. We recall this algorithm as
Algorithm~\ref{alg_sg}. 

\setcounter{algocf}{1}

\begin{algorithm}
  \caption{Algorithm from Textor et al.\ to find faithful DAGs for sets of marginal independencies~\cite{Idel15}.}
  \label{alg_sg}
  \DontPrintSemicolon
  \SetKwInOut{Input}{input}\SetKwInOut{Output}{output}
  \Input{Vertex set $V$, a DAG-representable set $\mathcal{I}_V^{0}$ of marginal independence statements}
  \Output{CPDAG $H$ which contains every faithful DAG and whose
    extensions are faithful}
  Form the 
  graph $H$ on the vertex set $V$ and empty edge set and the graph
  $\mathcal{U}$ which has an 
  edge $a-b$ if $(a \nindep
  b)_{\mathcal{I}_V^{0}}$. \; \label{init_line}
  \ForEach{edge $u-v$ in $\mathcal{U}$}{
    Add the edge $u \rightarrow v$ to $H$ if $\mathrm{Bd}_{\mathcal{U}}(u) \subset
    \mathrm{Bd}_{\mathcal{U}}(v)$. \; \label{edge_utov_line}
    Add the edge $u \leftarrow v$ to $H$ if $\mathrm{Bd}_{\mathcal{U}}(u) \supset
    \mathrm{Bd}_{\mathcal{U}}(v)$. \; \label{edge_vtou_line}
    Add the edge $u - v$ to $H$ if $\mathrm{Bd}_{\mathcal{U}}(u) = \mathrm{Bd}_{\mathcal{U}}(v)$. \; \label{end_adding_line}
  }
  {\bf Return} $H$
\end{algorithm}
We note that the boundary $\text{Bd}(i)$ is defined as the
neighborhood of node $i$ including $i$: $\text{Bd}(i) = N(i) \cup
\{i\}$ with the neighborhood $N(i)$ being the set of all nodes
adjacent to $i$. 

By showing that
Algorithm~\ref{alg_sg} produces the same
result as Algorithm~\ref{alg_gen_pdag} with parameter $k=0$, we formally
prove the correctness of the former algorithm (and prove all
additional properties, e.g.\ that the result is a CPDAG) and thereby
give the proof which has been missing from the literature. Note that this section does not
produce new results, but gives only the missing correctness proof. We include it for the sake of completeness.

\begin{thm}
  \label{alg_eq}
  For every set $\mathcal{I}_V^{0}$ of marginal independencies
  Algorithm~\ref{alg_gen_pdag} produces the same PDAG as
  Algorithm~\ref{alg_sg}.
\end{thm}

\begin{proof}
  In this proof we denote the PDAG produced by Algorithm~\ref{alg_sg}
  as $H$ and the one produced by Algorithm~\ref{alg_gen_pdag} as
  $G$. We have to show that $G = H$ holds. But before, we show that
  $G_{\mathrm{ep}} = H$ holds.

  We will begin our proof by analyzing
  under which conditions a directed edge $u \leftarrow v$ is removed
  from $G_{\mathrm{ep}}$ in the for-loop from
  line~\ref{start_for_allk} to~\ref{done_line}. We describe this
  through properties of the graph which was formed in
  line~\ref{alg_gen_start_line}. We call this graph $\mathcal{U}$ as
  it is the same as the graph formed in line~\ref{init_line} of
  Algorithm~\ref{alg_sg}.  A directed edge $u \leftarrow v$ is removed if we have a node $w$ which is a
  neighbor of $v$, but not a neighbor of $u$ in $\mathcal{U}$. Because then we have
  $u - v - w$ with $u$ and $w$ nonadjacent and in particular the edge
  $u \leftarrow v$ is removed.

  Formally, the edge $u
  \leftarrow v$ (in case we have $(u \nindep v)\in \mathcal{I}_V^0$)
  is removed from $G_{\mathrm{ep}}$ if the following condition holds:
  \begin{align}
    (\exists w) (w \not\in \text{Bd}_{\mathcal{U}}(u) \land w \in \text{Bd}_{\mathcal{U}}(v))
    \label{fcond}
  \end{align}
  We will also consider under which condition an edge $u \rightarrow
  v$ is not removed from $G_{\mathrm{ep}}$. This happens if condition~\ref{fcond} does not
  hold and by negation we get:
  \begin{align}
    &\relphantom{\Longleftrightarrow}{} \neg ((\exists w) (w \not\in \text{Bd}_{\mathcal{U}}(u) \land w
      \in \text{Bd}_{\mathcal{U}}(v))) \\
    &\Longleftrightarrow (\forall w) \neg (w \not\in \text{Bd}_{\mathcal{U}}(u)
                             \land w \in \text{Bd}_{\mathcal{U}}(v)) \\
    &\Longleftrightarrow (\forall w) (w \in \text{Bd}_{\mathcal{U}}(v) \implies
                             w \in \text{Bd}_{\mathcal{U}}(u))
                             \label{scond}
  \end{align}
  Now we can show that $H$ and $G_{\mathrm{ep}}$ are
  identical. We note that both graphs have the same vertex set. Thus,
  it is left to prove that all edges are identical. To do this we
  consider all possible edge states (undirected, directed or missing)
  between two node $u$ and $v$ in the following case study.
  \begin{enumerate}
  \item There is no edge between $u$ and $v$ in $G_{\mathrm{ep}}$ and $(u \indep
    v)_{\mathcal{I}_V^{0}}$ holds. Then, in the first line of both algorithms the edge was not
    added to $\mathcal{U}$ and thus is neither part of
    $H$ nor $G_{\mathrm{ep}}$. 
  \item The directed edge $u \rightarrow v$ is in $G_{\mathrm{ep}}$. From above
    considerations it follows that the conditions
    \[
      (\exists w) (w \not\in \text{Bd}_{\mathcal{U}}(u) \land w \in \text{Bd}_{\mathcal{U}}(v))
    \]
    and
    \[
      (\forall w) (w \in \text{Bd}_{\mathcal{U}}(u) \implies w \in \text{Bd}_{\mathcal{U}}(v))
    \]
    hold. This is because we require that the edge $u \leftarrow v$ was
    removed from $G_{\mathrm{ep}}$ while $u \rightarrow v$ was not. Only then we
    have the directed edge $u \rightarrow v$ in $G_{\mathrm{ep}}$.
    Moreover, it is clear that
    the edge between $u$ and $v$ is present in $\mathcal{U}$ in both algorithms.
    We can see that the two conditions above are equivalent to
    $\text{Bd}_{\mathcal{U}}(u) \subset \text{Bd}_{\mathcal{U}}(v)$ which is exactly the
    condition in line~\ref{edge_utov_line} in Algorithm~\ref{alg_sg}
    for adding an edge $u \rightarrow v$ to $H$.
  \item The directed edge $u \leftarrow v$ is in $G_{\mathrm{ep}}$. This case can be dealt
    with in the same way as $u \rightarrow v$ in case 2.
  \item \label{incomp_case}
    There is no edge between $u$ and $v$ in $G_{\mathrm{ep}}$ and $(u \nindep
    v)_{\mathcal{I}_V^{0}}$. In Algorithm~\ref{alg_gen_pdag} this case occurs if
    the edges $u \rightarrow v$ and $u \leftarrow v$ are removed from
    $G_{\mathrm{ep}}$ in different iterations in line~\ref{alg_gen_dir_line}.
    Thus, the following two conditions hold:
    \[
      (\exists w) (w \not\in \text{Bd}_{\mathcal{U}}(u) \land w \in \text{Bd}_{\mathcal{U}}(v))
    \]
    and
    \[
      (\exists x) (x \in \text{Bd}_{\mathcal{U}}(u) \land x \not\in \text{Bd}_{\mathcal{U}}(v)).
    \]
    This means that none of the three cases from
    line~\ref{edge_utov_line} to~\ref{end_adding_line} in
    Algorithm~\ref{alg_sg} apply as $\text{Bd}_{\mathcal{U}}(u)$ and
    $\text{Bd}_{\mathcal{U}}(v)$ are not equal nor is one a subset of the
    other. This means that no edge is
    added to $H$. We note here that the opposite direction holds as
    well, meaning that if none of the three cases apply it
    follows that the two statements above concerning the existence of $w$ and $x$
    are valid. Moreover, the nodes $u$ and $v$ are incompatible as
    \[
      (w \indep u)_{\mathcal{I}_V^0}, \: (w \nindep v)_{\mathcal{I}_V^0}, \: (v \nindep u)_{\mathcal{I}_V^0}, \: (x \indep
      v)_{\mathcal{I}_V^0} \text{ and } (x \nindep u)_{\mathcal{I}_V^0}
    \]
    hold. 
  \item There is an edge $u - v$ in $G_{\mathrm{ep}}$.
    This can only occur in Algorithm~\ref{alg_gen_pdag} if neither $u
    \rightarrow v$ nor $u \leftarrow v$ get removed. But this means as
    reasoned above that
    \[
      (\forall x) (x \in \text{Bd}_{\mathcal{U}}(u) \implies x \in \text{Bd}_{\mathcal{U}}(v))
    \]
    and
    \[
      (\forall x) (x \in \text{Bd}_{\mathcal{U}}(v) \implies x \in \text{Bd}_{\mathcal{U}}(u))
    \]
    hold. It immediately follows that $\text{Bd}_{\mathcal{U}}(u) = \text{Bd}_{\mathcal{U}}(v)$
    and therefore the edge $u - v$ is added to $H$ in
    line~\ref{end_adding_line} of Algorithm~\ref{alg_sg}.
  \end{enumerate}

  We can conclude that $G_{\mathrm{ep}} = H$ holds. We will show now
  that the Meek rules which are applied to $G_{\mathrm{ep}}$ will not
  direct further edges.
  \begin{enumerate}
    \item The first Meek rule states that an edge $b - c$ is oriented
      as $b \rightarrow c$ if we have $a \rightarrow b - c$ with $a$
      and $c$ nonadjacent. An analysis similar to the one made in the
      proof of Proposition~\ref{em_vs_thm} yields that whenever two nodes
      $a$ and $c$ are nonadjacent in $G_{\mathrm{ep}}$, every possible chain $a - x - c$ is
      directed as either $a \rightarrow x \leftarrow c$ or $a
      \leftarrow x \rightarrow c$.  Therefore the structure $a
      \rightarrow b - c$ can never appear. Note that this holds only
      for sets of \emph{marginal} independencies $\mathcal{I}_V^0$.
    \item The second Meek rule states that an edge $a - c$ is oriented
      as $a \rightarrow c$ if we have $a \rightarrow b \rightarrow
      c$. In $H$ (and thereby also in $G_{\mathrm{ep}}$ as these graphs
        are identical) we have an edge $a \rightarrow b$ iff
        $\text{Bd}_{\mathcal{U}}(a) \subset
        \text{Bd}_{\mathcal{U}}(b)$ and $b \rightarrow c$ iff $\text{Bd}_{\mathcal{U}}(b) \subset
        \text{Bd}_{\mathcal{U}}(c)$. It follows that
        $\text{Bd}_{\mathcal{U}}(a) \subset
        \text{Bd}_{\mathcal{U}}(c)$ holds as well meaning the edge
        between $a$ and $c$ is already oriented $a \rightarrow c$.
      \item The third Meek rule states that an edge $a - b$ is
        oriented into $a \rightarrow b$ whenever there are two chains
        $a - c \rightarrow b$ and $a - d \rightarrow b$ such that $c$
        and $d$ are nonadjacent. We know from the analysis of the
        first Meek rule that a structure $c - a - d$ with $c$ and $d$
        nonadjacent will never occur in $G_{\mathrm{ep}}$. Thus, the
        third Meek rule will never be applied as well.
      \end{enumerate}

      It follows that $G = G_{\mathrm{ep}} = H$.
\end{proof}

\end{document}